\documentclass[letterpaper]{article}

\usepackage[nonatbib, preprint]{neurips_2020}

\usepackage[utf8]{inputenc} 
\usepackage[T1]{fontenc}    
\usepackage{hyperref}       
\usepackage{url}            
\usepackage{booktabs}       
\usepackage{amsthm}
\usepackage{amsfonts}       
\usepackage{nicefrac}       
\usepackage{microtype}      
\usepackage[acronym]{glossaries}
\usepackage{color}
\usepackage{xspace}
\usepackage{url}
\usepackage{tikz}
\usepackage{tikz-qtree}
\usepackage{forest}
\usepackage{subfig}
\usepackage{graphicx}
\usepackage[noend]{algpseudocode}
\usepackage[linesnumbered,ruled,vlined]{algorithm2e}
\usepackage[inline]{enumitem}
\usepackage{cleveref}
\usepackage{nicefrac}
\usepackage{float}

\usepackage{amsmath,amsfonts,bm}

\newtheorem{theorem}{Theorem}[section]

\newtheorem{lemma}[theorem]{Lemma}
\newtheorem{definition}[theorem]{Definition}

\def\eqref#1{equation~\ref{#1}}

\def\1{\bm{1}}

\def\eps{{\varepsilon}}

\def\vdelta{{\boldsymbol \delta}}
\def\vxi{{\boldsymbol \xi}}
\def\vtau{{\boldsymbol \tau}}

\def\vzero{{\bm{0}}}

\def\vzero{{\mathbf{0}}}

\def\ve{{\mathbf{e}}}

\def\vl{{\mathbf{l}}}

\def\vu{{\mathbf{u}}}

\def\vx{{\mathbf{x}}}

\def\vz{{\mathbf{z}}}

\def\mX{{\mathbf{X}}}

\DeclareMathAlphabet{\mathsfit}{\encodingdefault}{\sfdefault}{m}{sl}
\SetMathAlphabet{\mathsfit}{bold}{\encodingdefault}{\sfdefault}{bx}{n}

\def\gA{{\mathcal{A}}}

\def\gD{{\mathcal{D}}}

\def\gP{{\mathcal{P}}}

\newcommand{\E}{\mathbb{E}}

\newcommand{\R}{\mathbb{R}}

\newcommand{\cm}{\textsf{CountMin}}
\newcommand{\union}{\cup}

\newcommand{\iforest}{iForest}
\newcommand{\rrcf}{RRCF}
\newcommand{\pid}{PIDForest}

\newcommand{\Polya}{P{\'o}lya}
\newcommand{\MPT}{Mondrian \Polya{} Tree}
\newcommand{\PTree}{\Polya{} Tree}

\newcommand{\BetaDist}[2]{\text{Beta}\left(#1,#2\right)}
\newcommand{\ExpDist}[1]{\text{Exp}\left( #1 \right)}
\newcommand{\cdf}[1]{\textsf{CDF} \left( #1 \right)}

\newcommand{\Troot}{\epsilon}
\newcommand{\tree}{\mathsf{T}}
\newcommand{\leaves}[1]{\textsf{leaves}\left( #1 \right)}
\newcommand{\cut}{\xi}
\newcommand{\cdim}{\delta}
\newcommand{\TreePath}{\textsf{path}}
\newcommand{\Child}[1]{\textsf{child}\left( #1 \right)}
\newcommand{\lChild}[1]{\textsf{left}\left( #1 \right)}
\newcommand{\rChild}[1]{\textsf{right}\left( #1 \right)}
\newcommand{\sib}[1]{#1_{\textsf{sib} } }
\newcommand{\parent}[1]{\textsf{parent} \left( #1 \right)}
\newcommand{\CutDims}{\vdelta}
\newcommand{\CutLocs}{\vxi}
\newcommand{\CutTimes}{\vtau}
\newcommand{\MTree}[2]{\textsf{MT}\left( #1, #2 \right)}
\newcommand{\MondPolyaTree}[2]{\textsf{sMPT}\left( #1, #2 \right)}
\newcommand{\MPTreeInsert}[2]{\textsf{sMPT}_{+}\left( #1, #2 \right)}
\newcommand{\MPTreeDelete}[2]{\textsf{sMPT}_{-}\left( #1, #2 \right)}
\newcommand{\vol}[1]{\text{vol}\left( #1 \right)}
\newcommand{\Cut}[1]{\texttt{cut}\left(#1 \right)}
\newcommand{\Restrict}[1]{\texttt{restrict}\left(#1 \right)}
\newcommand{\bbox}[1]{\textsf{box}\left(#1\right)}
\newcommand{\LinDim}[1]{\textsf{LinearDim}\left(#1\right)}
\newcommand{\TreeLeft}{\textsf{left}}
\newcommand{\TreeRight}{\textsf{right}}
\newcommand{\TreeDepth}[1]{\textsf{depth}\left( #1\right)}
\newcommand{\PolyaTreeDepth}[1]{\textsf{PolyaDepth}\left( #1\right)}

\graphicspath{{./figures/}}
\newacronym{1csvm}{$1$cSVM}{One Class SVM}
\newacronym{lof}{LOF}{Local Outlier Factor}
\newacronym{knn}{$k$NN}{K Nearest Neighbours}
\newacronym{pca}{PCA}{Principal Components Analysis}
\newacronym{svd}{SVD}{Singular Value Decomposition}
\newacronym{mp}{MP}{Mondrian Process}
\newacronym{mf}{MF}{Mondrian Forest}
\newacronym{pf}{PF}{\Polya{} Forest}
\newacronym{mpt}{sMPT}{Streaming Mondrian \Polya{} Tree}
\newacronym{bmpt}{bMPT}{Batch Mondrian \Polya{} Tree}
\newacronym{mpf}{MPF}{Mondrian \Polya{} Forest}
\newacronym{bmpf}{bMPF}{Batch Mondrian \Polya{} Forest}
\newacronym{smpf}{sMPF}{Streaming Mondrian \Polya{} Forest}
\newacronym{pyod}{PyOD}{Python Outlier Detection}
\newacronym{auc}{AUC}{area under curve}

\Crefname{equation}{Eq.}{Eqs.}
\Crefname{figure}{Fig.}{Figs.}
\Crefname{section}{Sec.}{Secs.}
\Crefname{algorithm}{Alg.}{Algs.}
\Crefname{theorem}{Thm.}{Thms.}
\Crefname{appendix}{Appx.}{Appxs.}
\Crefname{table}{Tab.}{Tabs.}
\Crefname{lemma}{Lemma}{Lemmas}
\Crefname{definition}{Def.}{Defs.}

\tikzset{every tree node/.style={minimum width=2em,draw,circle},
         blank/.style={draw=none},
         edge from parent/.style=
         {draw, edge from parent path={(\tikzparentnode) -- (\tikzchildnode)}},
         level distance=1.5cm}

\title{Interpretable Anomaly Detection with Mondrian \Polya{} Forests on Data Streams}

\author{%
  Charlie Dickens\thanks{Work done while an intern at Amazon Cambridge} \\
  University of Warwick\\
  \And
  {Eric Meissner\thanks{Work done while at Amazon Cambridge}} \\
  {University of Cambridge} \\
  \And
  {Pablo G. Moreno} \\
  {Amazon} \\
  \And
  {Tom Diethe} \\
  {Amazon} \\
}

\begin{document}

\maketitle

\begin{abstract}
Anomaly detection at scale is an extremely challenging problem of great practicality. When data is large and high-dimensional, it can be difficult to detect which observations do not fit the expected behaviour. 
Recent work has coalesced on variations of (random) $k$\emph{d-trees}
to summarise data for anomaly detection. 
However, these methods rely on ad-hoc score functions that are not easy to interpret, making it difficult to asses the severity of the detected anomalies or select a reasonable threshold in the absence of labelled anomalies.
To solve these issues, we contextualise these methods in a probabilistic framework which we call the Mondrian \Polya{} Forest for estimating the underlying probability density function generating the data and enabling greater interpretability than prior work.
In addition, we develop a memory efficient variant able to operate in the modern streaming environments. 
Our experiments show that these methods achieves
state-of-the-art performance while providing statistically interpretable anomaly scores.

\end{abstract}

\section{Introduction} \label{sec: intro}

The growing size of modern machine learning deployments necessitates automating
certain tasks within the entire pipeline from data collection to
model usage.
A key facet of this process at industrial scale is deciding on which data to
fit models.
Broadly, one can think of this as a subprocess in a continual learning
environment in which an algorithm should be able to return \emph{anomalies}
(points which do not conform to the behaviour of the rest of the dataset) and
monitor \emph{distribution or concept shift} \cite{diethe2019continual}.
Ideally, such a process would flag such anomalous points, along with
some information which enables interpretability to the user.

However, due to the scale and dimensionality of modern data,
building models 
for anomaly detection can often be difficult.
Often, storing or accessing an entire dataset at once is not possible, driving our interest in
the so-called \emph{streaming model} of computation.
Here, data $\mX \in \R^{n \times D}$ is assumed to be too large to hold in 
memory so observations $\vx_i \in \R^D$ are accessed sequentially.
Additionally, the stream is \emph{dynamic}, so that new points may be added and removed from $\mX$ over time.
To answer queries of the data, it is permissible to store
a \emph{small space summary} of $\mX$ which is 
typically constructed using only one full pass over $\mX$.
While the streaming model is reminiscent of an online machine learning model, there are subtle differences, namely, the desire to delete 
data from the model.

Given this problem setting we strive to design an anomaly detector which
satisfies the following requirements:
\begin {enumerate*} [label=(\roman*)]
\item{The data is so large that only a small-space summary
can be retained, built in a single-pass over $\mX$}; 
\item{The summary should permit the insertion and 
deletion of datapoints}; 
\item{Anomalies must be declared in the unsupervised setting and}
\item{The user should be able to understand why points are 
flagged as anomalies i.e. the results are 
\emph{interpretable}.} 
\end {enumerate*} 

Existing solutions to the unsupervised anomaly detection 
problem have coalesced on random $k$d trees known as 
\emph{Isolation Forest} (\iforest) \cite{liu2008isolation},
\emph{Robust Random Cut Forest} (\rrcf) 
\cite{guha2016robust}, and \emph{PiDForest} 
\cite{gopalan2019pidforest}.
A problem common to all of these is the issue of interpretability:
each method introduces their own vague heuristic `scoring'
mechanism to declare anomalies which can make it difficult to understand why points are flagged as anomalous.
Both \iforest{} and \rrcf{} cut the input domain at random
which does not guarantee good partitioning of the space.
In addition, \iforest{} and \pid{} are fixed data structures
which may not well adapt to local or temporal changes 
in behaviour, a likely scenario on large data streams,
as observed for \iforest{} in \cite{guha2016robust}.
A particular issue for \pid{} is that the cuts are optimised
\emph{deterministically} for the given subsample of $\mX$.
In practise, we find this process to be slower than all 
other methods, but more generally, this could be problematic
when the data is dynamic and cuts need adjusting or updating
depending on behavioural changes.

\textbf{Contributions}
We present the Mondrian \Polya{} Forest (\acrshort{mpf}),
a probabilistic anomaly detection algorithm that combines random trees with
nonparametric density estimators.
This leads to a full Bayesian nonparametric model providing reliable estimates of low probability regions without making strong parametric (distributional) assumptions. 
Moreover, anomalies are declared in the probability domain; thus our method is inherently interpretable and avoids heuristic scores needed in previous algorithms based on random trees.
As a second contribution, we present an extension amenable to streaming scenarios (Streaming Mondrian \Polya{} Forest (\acrshort{smpf})) by proposing
two-level modification of the Mondrian Forest that can be seen as a
probabilistic extension of the well known \rrcf{} algorithm.
The proposed data structure can be efficiently implemented on a data stream, which enables speed and scalability. 
Along the way, we answer questions raised in
\cite{lakshminarayanan2016decision} and \cite{balog2015mondrian}, 
concerning the use of our proposed trees for anomaly detection and density estimators, respectively.

\begin{figure}
    \centering
    \includegraphics[width=0.9\linewidth]{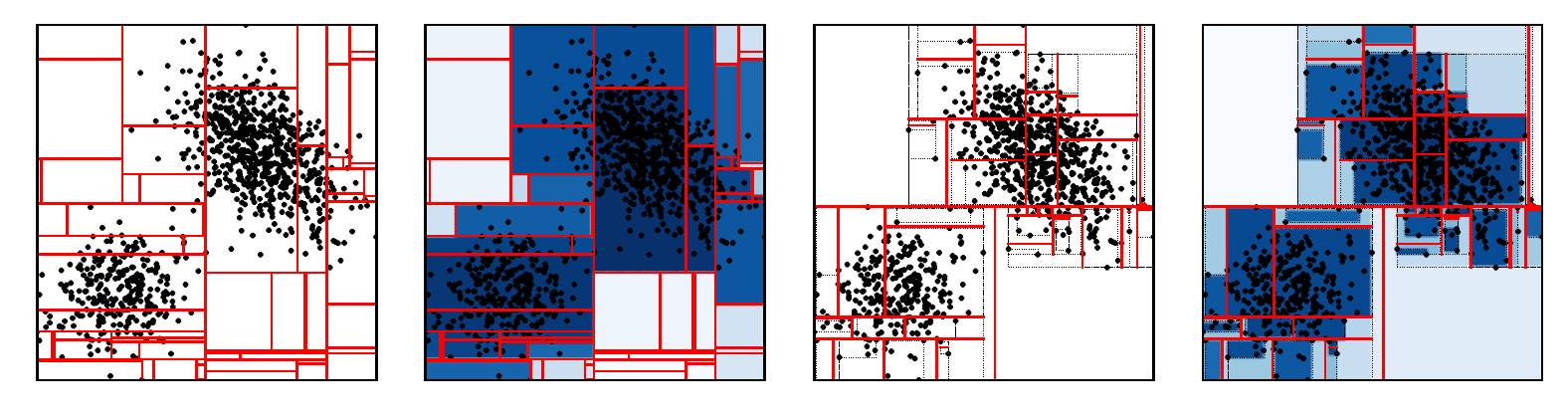}
    \caption{Our methods, the Mondrian \PTree{s}, are introduced
    in \Cref{sec: mondrian-polya-forest} and enable efficient space 
    partitioning with density estimation which we adopt for anomaly 
    detection.
    Denser regions are denoted by darker cells and results are
    averaged over a forest.
    Left to right: Mondrian Process, (batch) Mondrian
    \Polya{} Tree, Mondrian Tree, (streaming) Mondrian \Polya{} Tree.
    }
    \label{fig:all-mondrians}
\end{figure}
\textbf{Outline.}
\Cref{sec: preliminary} reviews prior work
from the Bayesian nonparametric literature on 
Decision, Mondrian \& \Polya{} Trees.
Our proposal, the \MPT{} is described in 
\Cref{sec: mondrian-polya-forest}.
Related work is reviewed in \Cref{sec: related-work} followed by
experiments in \Cref{sec: anomaly-detection}.
Proofs \& algorithms are in the Appendix.

\section{Preliminaries} \label{sec: preliminary}

We follow the notation introduced in
\cite{lakshminarayanan2014mondrian}. Given a fixed bounded domain $\gD \subset \R^D$,  a \textbf{decision tree}  over
$\gD$ is a hierarchical, nested, binary partition represented by a set of \textbf{nodes} $T$.
Every node $j$ has exactly one \textbf{parent} node $\parent{j}$ (with the
exception being the root node $\Troot$ which does not have a parent) and has
either 2 children if $j$ is an \textbf{internal} node or has 0 children if
$j$ is a \textbf{leaf} node;
The set of leaves is denoted $\leaves{\tree}$; 
(iii) To every node $j$ is associated a \textbf{subdomain} or \textbf{region} of the
input space $\gD$ denoted $B_j$; 
(iv) If $j$ is not a leaf, then the children of $j$ are constructed by making
a cut $\cut_j$ in dimension $\delta_j \in \{1,\dots,D\}$.
The children are $\lChild{j}$ and $\rChild{j}$ with
$\lChild{j}$ denoting the node which contains
the space $B_{\lChild{j}} = \{ \vx \in B_j : \vx_{\delta_j} \le \cut_j$ and
 $B_{\rChild{j}} = \{ \vx \in B_j : \vx_{\delta_j} > \cut_j \}$. 
The tuple $(\tree, \CutDims, \CutLocs)$ is a \textbf{decision tree}.

\subsection{Mondrian Processes \& Mondrian Forest}
\textbf{Mondrian Processes}
are families of (potentially infinite) hierarchical binary
partitions of a subdomain $\gD \subseteq \R^D$;
they can be thought of as a family of $k$d trees with height $h$, which
sequentially refine the partition of $\gD$ as $h$ increases 
\cite{roy2008mondrian}.
A Mondrian Tree can be defined as a restriction of the underlying 
Mondrian Process to an observed set of data points 
\cite{lakshminarayanan2014mondrian}.
Unlike the Mondrian Process, it allows for the online sampling of the 
stored tree as more data is observed.
Specifically, a \textbf{Mondrian Tree} $T$ can be represented by the tuple
$(\tree, \CutDims, \CutLocs, \CutTimes)$ for a decision tree
$(\tree, \CutDims, \CutLocs )$ whose cut dimensions $\CutDims$ are 
chosen with probability proportional to the feature lengths of data 
stored in a node
and $\CutTimes$ is a 
sequence of cut times
$\CutTimes = (\tau_j)_{j \in \tree}$
which begin from 0 at the root ($\tau_{\Troot}$) while monotonically increasing up to a \textbf{lifetime budget} $\lambda > 0$. 
For any node $j$, the \textbf{time} or \textbf{weighted depth} is the value
$\tau_j$, whereas  the \textbf{(absolute) depth} is the length of the
(unweighted) path from the root to $j$.
Given \emph{observations} $\mX$, the generative process for sampling Mondrian
Trees is denoted $\MTree{\mX}{\lambda}$.
For every node $j \in \tree$, the indices of the data stored at $j$ is 
denoted $N(j)$ (so we clearly have $N(\Troot) = \{1,\ldots, n\}$) and the 
regions of space a every node $B_j$ are the \emph{minimal axis-aligned box}
containing the data $\mX_{N(j)}$.
Additionally, the dimension-wise minima and maxima over 
$\mX_{N(j)}$ are stored in the vectors $\vl_j^{\mX}$ and 
$\vu_j^{\mX}$.
An example implementation is given in \Cref{alg: mondrian-forest}.

Mondrian Trees are attractive models as they can be sampled \emph{online}
as new data is observed. The key principle for this is \textbf{projectivity}, meaning that if 
$T \sim \MTree{\mX}{\lambda}$ and $\mX'$ is a subset of the data from $\mX$,
then the tree restricted to the datapoints $\mX'$  is drawn from  
$\MTree{\mX'}{\lambda}$ \cite{lakshminarayanan2014mondrian}.
Crucially, this enables the sequential building of Mondrian Trees:

\begin{lemma}[Projectivity] \label{lem:projectivity}
Let $\mX =  \{\vx_i\}_{i=1}^n, \mX' = \mX \cup \vx_{n+1}$.
Suppose $\textsf{MTx}(\mX',\lambda)$ is a random
function to extend the tree $T$.
  If $T \sim \MTree{\mX}{\lambda}$ and 
  $T' | T, \mX' \sim \textsf{MTx}(\mX',\lambda)$ then 
  $T' \sim \MTree{\mX'}{\lambda}$.
\end{lemma}

Hence, Mondrian \emph{Trees} are essentially, finite, truncated versions of 
Mondrian Processes in the regions of $\R^D$ where data is observed.
An ensemble of trees each independently sampled from $\MTree{\mX}{\lambda}$ is referred
to as a \textbf{Mondrian Forest}.

\subsection{\PTree}  \label{sec: ptree-intro}
The \PTree{} is a nonparametric model for estimating the density function 
over a nested binary partition of a bounded input domain, $\gD$.
We require the \PTree{} to decide how to distribute mass about the space
represented according to a random binary partition that we will sample.
First we will introduce the infinite version of the \PTree{} and then
demonstrate a restricted, finite \PTree{}
(further details can be found in \cite{muller2013polya}).

Suppose $\Pi_m =  \{ A_{b(j)}: j = 1,\dots,2^m\} $ is the depth $m$ partition of $\gD$
into $2^m$ disjoint subsets $j$, indexed by the binary
string $b(j) = e_0 e_1 \dots e_{m-1}$.
If we refine $\Pi_{m}$ to $\Pi_{m+1}$ by splitting every 
$A_{b(j)} = A_{b(j)0} \cup A_{b(j)1} $ with 
$A_{b(j)0} \cap A_{b(j)1} = \emptyset$ to generate 
then it remains to understand how mass is allocated to all subsets in 
$\Pi_{m+1}$.
The \PTree{} treats probability mass as a random variable which is 
distributed throughout $\Pi_m$ through split 
probabilities $\pi_{b(j)} \sim \BetaDist{\alpha_{b(j)0}}{\alpha_{b(j)1}}$,
each $\pi_{b(j)}$ being sampled independently across all levels of 
refinement, $m$.
The probability $\pi_{b(j)}$ is the probability of reaching the 
``right-hand side'' of the split: that is, choosing a point that is 
in $A_{b(j)1}$ given that the point is in $A_{b(j)}$.
Overall, the \PTree{} has two sets of parameters: the nested partition
$\Pi = \{ \Pi_m : m \ge 0 \}$ and the Beta distribution parameters
$\gA = \{ (\alpha_{b(j)0}, \alpha_{b(j)1} ) : j = 1,\dots, 2^m \}$.

A \textbf{\PTree{} over infinite depth partition} allows 
$m \rightarrow \infty$ and 
is capable of modelling absolutely continuous functions if the
$\alpha_{b(j)} = \Theta(m^2)$ or discrete functions if 
$\alpha_{b(j)} = \Theta(2^{-m})$.
Rather than let $m \rightarrow \infty$ 
a \textbf{\PTree{} over a finite depth partition} assumes the partition
is truncated at some fixed $m$.
Probability mass is then assumed to be distributed uniformly 
within the final $2^m$ bins.
An implementation of the \PTree{} is given in \Cref{alg: polya-tree-sampling} when the 
partition $\Pi_m$ is defined by a binary tree of height $m$,
as opposed to the online setting.
The predicitve distribution for density estimation over a finite 
partition is the product of expectations of the Beta distributions along the 
leaf-to-root path \cite{muller2013polya}.

\section{Mondrian \Polya{} Forest{}} \label{sec: mondrian-polya-forest}
Our contributions combine the \PTree{} structure with either a 
finite (truncated) Mondrian Process which operates in a batch 
setting or a Mondrian Tree which can be maintained over a data stream.
We then construct a forest using these revised trees which 
estimate the density function \& perform anomaly detection.
Using Mondrian Trees for anomaly detection was mentioned in
\cite{lakshminarayanan2016decision} and density estimation in
\cite{balog2015mondrian}, however, no feasible solutions were offered
so our alterations answer these unresolved questions.
Our methods are referred to as batch or streaming Mondrian \PTree{s}
and a visual comparison is given in \Cref{fig:all-mondrians}.
First we will introduce the batch solution.

\textbf{Batch Mondrian (Process) \PTree{} (\acrshort{bmpt})}
Let $T \sim MP(\gD, \lambda)$ denote the binary tree sampled from the \emph{Mondrian Process} with lifetime $\lambda > 0$.\footnote{Note that this is \emph{not} a ``Mondrian Tree'' as defined in 
\Cref{sec: preliminary}!}
A 
\acrshort{bmpt} is the combination of 
$T$ with the \PTree{} density model.
For every node $j$ in the tree, the prior Beta parameters $\alpha_{b(j)}$
can be computed exactly 
from the volume of every node and incremented by the number of points in
$j$ to obtain the posterior parameters.
We drop the ``process'' \& refer to this method as  ``batch Mondrian \Polya{} Tree''.

Since the Mondrian Process (\acrshort{mp}) on a bounded domain fully accounts for the 
entire space, we can 
easily combine the \acrshort{mp} with the \PTree.
All subsets of the partition induced by the 
\acrshort{mp} are covered by a region where is 
\acrshort{mp} is instantiated.
Hence, all the volume computations necessary for 
the \PTree{} are well-defined.
However, combining the \PTree{} with the online version of 
the \acrshort{mp} (i.e. the Mondrian Tree) is much more challenging.
The alteration we make is necessary as nai\"vely imposing the \PTree{} prior over a Mondrian Tree 
would leave `empty' space across the domain as cuts are defined \emph{only}
on regions of space where data is observed.
The \PTree{} cannot handle this scenario as refining a bin
$A_{b(j)}$ into children $A_{b(j0)}, A_{b(j1)}$ requires that 
$A_{b(j0)} \cap A_{b(j1)} = \emptyset$ and 
$A_{b(j0)} \cup A_{b(j1)} = A_{b(j)}$ which is clearly false 
if we immediately restrict to the data either side of a cut.
This is an issue for density estimation as it is not clear how to assign mass to the regions
where data is not observed,
exactly the issue encountered in \cite{balog2015mondrian}.

A natural question is why use Mondrian Trees as opposed to Mondrian 
Processes?
There are two reasons:
firstly, Mondrian Processes are infinite structures so they cannot 
always be succinctly represented.
The restriction to a finite lifetime $\lambda$ does not 
guarantee that the tree is finite, so over $\R^D$, it would be possible
to have an infinitely deep tree with infinitely many leaves.
Secondly, in high-dimensional space, there could be many empty regions 
of space with no observed data.
A Mondrian Process may repeatedly cut in the empty regions 
yielding many uninformative cuts; thus, a very deep tree would be 
necessary.
On the other hand, Mondrian Trees focus cuts on the regions of space
where data is observed, which ensures that cuts are guaranteed to lie 
on a subset of the domain which will split the data.\footnote{This provides no guarantee on the quality of the cuts,
merely that they exist on the region of space where they will pass
through observed data with certainty.}
The price to pay for this advantage is that Mondrian Trees 
are unable to model data lying outside of the bounding box upon which 
they are defined.
This motivates our altered method, the streaming \MPT{} 
which combines the scalability of the 
Mondrian Tree generative process with an added twist to
cheaply model behaviour beyond the observed data.

\subsection{Streaming \MPT} \label{sec: mpt-intro}

\forestset{
 declare toks={elo}{}, 
 anchors/.style={anchor=#1,child anchor=#1,parent anchor=#1},
 dot/.style={tikz+={\fill (.child anchor) circle[radius=#1];}},
 dot/.default=2pt,
 decision edge label/.style n args=3{
 edge label/.expanded={node[midway,auto=#1,anchor=#2,\forestoption{elo}]{\strut$\unexpanded{#3}$}}
 },
 decision/.style={if n=1
 {decision edge label={left}{east}{#1}}
 {decision edge label={right}{west}{#1}}
 },
 decision tree/.style={
 for tree={
 s sep=1.0em,l=8ex,
 if n children=0{anchors=north}{
 if n=1{anchors=south east}{anchors=south west}},
 math content,
 },
 anchors=south, outer sep=4pt,
 dot=3pt,for descendants=dot,
 delay={for descendants={split option={content}{;}{content,decision}}},
 }
}

\begin{figure}%
    \centering
    \subfloat[An \acrshort{mpt} on $\gD = \bbox{\{\vx_1,\dots,\vx_4 \}}$.
    Strings in the top left of any box denote the encoding of that
    part of the space under the random partitioning; set
    notation corresponds to leaves of \Cref{fig:mpt-binary-tree}
    \label{fig:mpt-example-plot}]
    {\includegraphics[scale=0.9]{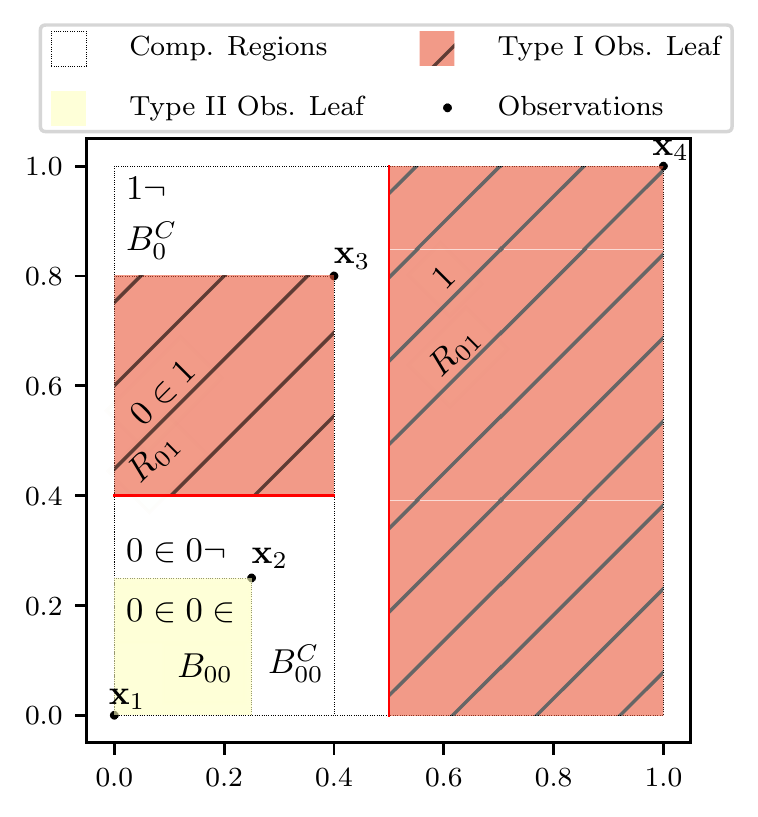}}%
    \quad
    \subfloat[The \acrshort{mpt} with parameters indicating mass distribution
    and space partition from \Cref{fig:mpt-example-plot}]{{
    \begin{forest} decision tree
 [{$\epsilon: P_{\epsilon} = 1.0$},plain content
 [{$R_{0}$};{\mu_{\chi_{\Troot}}},plain content,elo={yshift=4pt}
 [{$B_{0}$};{\mu_{\rho_0}},plain content
 [{R_{00}};\mu_{\chi_{00}}
 [{B_{00}};{\mu_{\rho_{00}}}]
 [{B_{00}^C};{1 - \mu_{\rho_{00}}}]
 ]
 [{R_{01}};{1 - \mu_{\chi_{00}}}]
 ]
  [{$B_{ 0}^C$};{\mu_{1 - \rho_0}},plain content]
 ]
 [{$R_{1}$};{1 - \mu_{\chi_{\Troot}}},plain content,elo={yshift=4pt}
 ]
 ] 
 ]
\end{forest} \label{fig:mpt-binary-tree}
    }}%
    \caption{A streaming \MPT{}.  From left to right, the leaf/data encodings are:
    $\{ (0\in0\in : \{\vx_{1},\vx_2 \}),
    \{ (0\in0\neg : \text{complementary leaf}),
    (0\in1 : \{\vx_3 \}),
    (1 \neg :  \text{complementary leaf}),
    (1 : \vx_4)\}$}
    \label{fig:mpt-example-tree}
    
\end{figure}

The standard Mondrian Tree \emph{only} considers (sub-)regions 
where data is observed: information about space 
without observations is discarded.
We decouple the splitting method used to generate Mondrian Trees into
a two-step procedure which will allow a tree sampled
in the Mondrian Tree $T \sim  \MTree{\mX}{\lambda}$ to implicitly 
represent the entirety of a given input domain so we can appeal to the 
\PTree{} model.

Our modified Mondrian Tree is the 
\textbf{streaming Mondrian \Polya{} Tree}
(\acrshort{mpt}) and draws from this structure are 
denoted $T \sim \MondPolyaTree{\mX}{\lambda}$.
We generate an \acrshort{mpt} by drawing $T \sim \MTree{\mX}{\lambda}$
and introducing `pseudosplits' in $T$ to generate a new, implicitly 
defined $T_s$.
These pseudosplits distinguish between the regions of space where data is
observed, and those which do not contain any data so the only extra
space cost we incur is that of storing the parameters for the Beta 
distributions necessary for the \PTree.
An example is illustrated in \Cref{fig:mpt-example-plot} and is 
implemented in \Cref{alg: mpt}.

Let $T = (\tree, \CutDims, \CutLocs, \CutTimes) \sim \MTree{\mX}{\lambda}$
be a Mondrian Tree.
We define two functions over nodes $j \in \tree$ to decouple the 
cutting of space from the restriction to bounding boxes.
First recall that every node $j$ has a minimal axis-aligned bounding box $B_j$:
(i) $\Cut{j}$ samples a cut dimension $\cdim_j$ and cut location $\cut_j$ splitting
$j$ into disjoint sets $R_{j0}, R_{j1}$ with $R_{j0} \cup R_{j1} = B_j$, and 
$R_{j0}, R_{j1}$ representing the \emph{region} of space less than cut $\xi_j$ and greater 
than $\xi_j$, respectively.
For $k=0,1$, the set $R_{j k}$ is a bounded region of space which is \emph{not  
restricted} to the observations located at the child nodes $\mX_{N(\Child{j})}$.
This motivates the subsequent `split' to maintain the Mondrian Tree 
structure of only performing $\Cut{\cdot}$ on bounding boxes of 
observed data; 
(ii) $\Restrict{j}$ acts on the pair $R_{j0}, R_{j1}$, returning 
$\left\{ \left( B_{\lChild{j}}, B_{\lChild{j}}^C \right), 
\left( B_{\rChild{j}}, B_{\rChild{j}}^C \right) \right\}$ such that 
$B_{\lChild{j}} \cup B_{\lChild{j}}^C = R_{j0}$ and are pairwise disjoint.
The same property holds for $R_{j1}$ with the right-hand child nodes.
We refer to $B_{\Child{j}}$ as the \textbf{observed region} and 
$B_{\Child{j}}^C$ as the \textbf{complementary region} where $\Child{j}$ can be 
either $\lChild{j}$ or $\rChild{j}$.
We term this as a `pseudosplit' because all of the information 
required to perform $\Restrict{\Cut{j}}$ is already defined in the generation 
of the Mondrian Tree.

\textbf{Combining \acrshort{mpt} with Finite \PTree{}.}
Decoupling the `cut-then-restrict' allows the Mondrian Tree $T$ 
to encode a valid hierarchical partition over the entirety of the input domain $\gD$.
Additionally, we only ever store the $T$ and the extra Beta parameters
from the \PTree{} structure as $T$
implicitly defines the \acrshort{mpt} $T_s$.
These Beta parameters are \textbf{cut parameters}: $\chi_{j0}, \chi_{j1}$, index 0 for less than $\xi_j$, 1 otherwise; 
\textbf{restriction parameters} 
$\rho_{j0 \in}, \rho_{j0 \neg}$ indexed by $j0\in$ for the observed region
$B_{\lChild{j}}$ and $j0\neg$ for the complementary region $B_{\lChild{j}}^C$
(and similarly for $\rho_{j1 \in}, \rho_{j1 \neg}$) at node $j$.

The following distinctions are necessary to ensure all volume
 comparisons for the \PTree{} construction are on $D$-dimensional 
 hypervolumes while the final distinction is necessary to account for 
 the mass associated to regions with no observations.
 \begin {enumerate*} [label=(\roman*)]
   \item{\textbf{Observation leaves (Type I)}:
   Any leaf $l$ for which the bounding box 
   at $l$, $B_l = \bbox{X_{N(l)}}$ has at least one of the dimensions with zero
   length.
   Note that this includes the case when only one datapoint is stored
   in $l$; 
   }

   \item{\textbf{Observation leaves (Type II)}:
   Any leaf $l$ formed from a cut at node $j$ which contains two or more datapoints and $B_l = \bbox{X_{N(l)}}$ has all $D$ lengths positive; 
  }

   \item{\textbf{Complementary leaves}: Leaves formed in a region where no
   observations are made.
    }
 \end{enumerate*}

Pseudosplits in the Mondrian Tree $T$ are used to generate the 
\acrshort{mpt}, so it is necessary to revise the indexing scheme of the
nested partition over domain $\gD$.
For every node $j \in \tree$ of (absolute) depth $k$ in
$T$, we generate a set of encodings for the spaces represented in $T_s$:
$b(j) = c_0 r_0 \dots c_k r_k$.
The length of $b(j)$ is at most twice the maximum absolute depth 
of $T$ and indexes all nodes in the (implicitly defined) $T_s$.
The symbol $c_i \in \{ 0,1 \}$ indicates ``less than'' or ``greater than'' the cut at level
$i$, and $r_i \in \{ \in, \neg, \emptyset \}$ indicates whether the node represents
the observed region ($\in$), complementary region ($\neg$), or can be $\emptyset$
if the leaf is a Type I observed leaf as no restriction is performed.

\subsubsection{Model Parameters for the \acrshort{mpt}} \label{sec: mpt-params}
For a Mondrian Tree $T = (\tree, \CutDims, \CutLocs, \CutTimes)$ we now
show how to set the parameters for the \PTree{} over the induced 
\acrshort{mpt}, $T_s$.
\begin{itemize}

\item{Let $P_j$ denote the probability mass associated with node $j$.
We define the probability mass associated with the root $\Troot$ to
be $P_{\Troot} = 1$.}

\item{\textbf{Setting the prior.}
At every internal node $j \in \tree$, the probability of a point
being greater than the cut $\xi_j$ is given by a Bernoulli 
with parameter $\kappa_j$ whose prior is 
$\BetaDist{\chi_{j0}}{\chi_{j1}}$.
Likewise, the probability of a point being in the observed 
region after cut $\xi_j$ follows a\footnote{Note that we choose this
ordering so the expected
value of the Beta distribution is associated with being in 
the observed region after every cut.
See example in \Cref{fig:mpt-example-tree}.}  $\text{Bernoulli}(1-\theta_j)$ whose prior is 
$\BetaDist{\rho_{j k \in}}{\rho_{j k \neg}}$
for $k=0,1$, depending on which side 
of the cut the point lies.
To set the \PTree{} parameters, we need to evaluate the volumes of 
various parts of the space, this will be denoted $V_X = \vol{X}$ 
for $X \in \{R_{\cdot}, B_{\cdot}, B_{\cdot}^C\}$.
The \acrshort{mpt} $T_s$ is defined from a two-stage split so we correct
the `depth' of nodes in $T_s$ from the usual \PTree{} construction by
a simple translation: if $j \in \tree$ has depth $d_j$ then
$\PolyaTreeDepth{j}  = (2d_j, 2d_j+1)$.
The prior strength is controlled by hyperparameter $\gamma > 0$.
Parameters for cut ($\chi_{\cdot}$)  and 
restriction ($\rho_{\cdot}$) are then:
\begin{equation} \label{eq: mpt-prior}
  \begin{split}
    \chi_{j0} &= \gamma (2d_j + 1)^2 
    \nicefrac{ V_{R_{j0}}}{ ( V_{R_{j0}} +  V_{R_{j1}} ) }\\
    \chi_{j1} &= \gamma (2d_j + 1)^2 
    \nicefrac{ V_{R_{j1}}}{ ( V_{R_{j0}} +  V_{R_{j1}} ) }\\
  \end{split}
\quad
  \begin{split}
    \rho_{j k \in} &= \gamma (2d_j + 2)^2 
    \nicefrac{ V_{B_{jk\in}}}{ ( V_{B_{jk\in}} +  V_{B_{jk\neg}} ) }\\
    \rho_{j k \neg} &= \gamma (2d_j + 2)^2 
    \nicefrac{ V_{B_{jk\neg}}}{ ( V_{B_{jk\in}} +  V_{B_{jk\neg}} ) }\\
  \end{split}
\end{equation}

}

\item{\textbf{Distributing Mass.}
The predictive distribution of the \PTree{} over a finite 
depth partition is the product of expected value of Beta distributions on the leaf-root path.
There can be maintained exactly over all nodes for both cutting \& restricting.
We allocate a 
$\mu_{\chi_j} = \E(\BetaDist{\chi_{j0}}{\chi_{j1}})$ fraction of $j$'s mass to $R_{j0}$ so $P_{j0} = P_j \mu_{\chi_j}$ \& 
$P_{j1} = P_j (1 - \mu_{\chi_j})$.
Next, we repeat for the restriction step which allots 
$\mu_{\rho_{j0}} = \E(\BetaDist{\rho_{j0\in}}{\rho_{j0\neg}})$ to $B_{j0}$ so  $P_{j0\in} = P_{j0} \mu_{\rho_{j0}}$ and 
$P_{j0\neg} = P_{j0} (1 - \mu_{\rho_{j0}})$; likewise for 
above the cut $\xi_j$.

}

\item{\textbf{The Posterior Distribution.}
By Beta-Binomial conjugacy, on inserting data, the parameters of any given
Beta distribution can be updated by the number of datapoints observed at a node.
In the Mondrian Tree, $n_0$ points are passed from $j$ to $\lChild{j}$
 and $n_1$ points to $\rChild{j}$.
 Hence, all of the $n_0$ points in $\lChild{j}$ are both at most the 
 cut value $\xi_j$ and present in the bounding box $B_{\lChild{j}}$ 
 while the opposite is true for $n_1$ and $\rChild{j}$.
 Therefore, we obtain the simple posterior update procedure:

\begin{equation} \label{eq: mpt-posterior}
  \chi_{jk}^* = \chi_{jk} + n_k, \quad
  \rho_{j k \in}^* = \rho_{j k \in} + n_k, \quad
    \rho_{j k \neg}^* = \rho_{j k \neg}, \quad \mbox{for }k=0,1
\end{equation}
}
%
%
\item{Mass in the leaves of a finite \PTree{} is assumed to be 
distributed uniformly;  if a point 
falls into a leaf $j$, then the mass associated to that point is 
simply the product of the expected Beta distributions on the path from 
$\Troot$ to $j$, and its density is the mass divided by the volume.}

\end{itemize}
It is necessary to retain the volumes of both observed and complementary regions for the restriction parameters.
However, this is straightforward given the cut and volume at node $j$ (see \Cref{sec: mpt-appendix}).


\textbf{Complexity: Instantiating the \acrshort{mpt}.}
The complexity of combining the \PTree{} with the Mondrian Tree incurs only mild overhead.
Let $T = (\tree, \CutDims, \CutLocs, \CutTimes)$ denote the stored Mondrian Tree which 
generates the \acrshort{mpt}.
The extra space necessary to use \acrshort{mpt} for density estimation is  $ 7 |\tree|$ due to the extra counters needed for every Beta distribution (i.e. $\chi_{j0},\chi_{j1},
\rho_{jk\in},\rho_{jk\neg}, k=0,1$) and the probability mass float $P_j$. 
At every node we must compute the volume at a cost
of $O(D)$ which is $O(D |\tree|)$ over the entire tree but this can be done on-the-fly as the tree is constructed.

\textbf{\acrshort{mpt}: Insertions and Deletions.}
For a \acrshort{mpt}, we provide efficient algorithms to insert and 
delete points over the data stream.
A full treatment is given in \Cref{sec: mpt-appendix}: the pertinent points
being that we retain \emph{projectivity} due to the underlying Mondrian 
Tree which generates the \acrshort{mpt}.
Deletions require a little more work as removal points could lie on a 
bounding box, so it is necessary to check how this interacts with the
lifetime of the stored tree.

\textbf{Example.}
In \Cref{fig:mpt-example-tree} we present an instantiation of the \acrshort{mpt}.
Observe that the \emph{Mondrian Tree} which is used to generate the partition
in \Cref{fig:mpt-example-plot} splits the entire input domain into 
disjoint subsets of Type I/II observed leaves and complementary regions.
It also implicitly encodes the associated \acrshort{mpt} as given in 
\Cref{fig:mpt-binary-tree}.
The calculations to obtain density estimates over this tree are given in \Cref{sec:mpt-example-calcs}.

 \subsection{Mondrian \Polya{} Forest for Density Estimation
 \& Anomaly Detection}
 \label{sec: mpf-anomaly}
 Recall that an independently sampled ensemble of 
 batch/streaming \MPT{s} is 
 referred to as a batch/streaming Mondrian \Polya{} Forest
 (\acrshort{bmpf}), (\acrshort{smpf}),
 $F = \cup_i T_i$.
 Each $T_i$ defines a function over its leaves $p_i(\vx)$
 which is a noisy estimate of the true underlying 
 density function $p(\vx)$.
 
 \begin{definition}[Density Estimation]
Let $p(\vx)$ be a density function and suppose 
 $F = \cup_i T_i$ is a \acrshort{bmpf} or \acrshort{smpf}.
 Let $l$ denote the leaf in $T_i$ which contains $\vx$
 and whose mass is $P^{(i)}_l$.
The \emph{density estimate} of $\vx$ in $T_i$ is 
$p_i(\vx) = P_l / \vol{l}$ while the density estimate over
the forest is $\hat{p}(\vx) = \frac1n \sum_{i=1}^{|F|}p_i(\vx)$.
 \end{definition}
Rather than using density estimates, we adopt the following simple approach to declare anomalies while 
remaining in probability space; using simply
the $P^{(i)}_l$ rather than $p_i(\vx)$.
This alteration is to prevent a small number of trees from 
corrupting the `score' if they are not good trees.

The simplicity of this approach is one of the strengths of 
our work.
While previous works add an extra scoring mechanism over 
the forest, ours is an inherent property of the underlying
probabilistic framework.
We can threshold exactly in probability space which makes 
these `scores' more interpretable than prior work.
Synthetic density estimation \& anomaly detection examples
are in \Cref{sec: synthetic-examples}.

\begin{definition}[$\eps$-anomaly \& $(\eps,\phi)$-anomaly]
\label{def: anomaly}
Let $F = \cup_i T_i$ be a \acrshort{bmpf} or \acrshort{smpf} and
$\eps, \phi \in [0,1]$.
A point $\vx \in \R^d$ is an \emph{$\eps$-anomaly} in tree $T$ 
if the probability mass of the leaf in which $\vx$ is stored is at most
$\eps$.
A point $\vx \in \R^d$ is an \emph{$(\eps,\phi)$-anomaly}
if $\vx$ is and $\eps$-anomaly in at least $\phi |F|$
trees from $F$.
\end{definition}

\section{Related Work} \label{sec: related-work}
Initiated by the success of the so-called \emph{isolation forest} (\iforest{})
\cite{liu2008isolation},
random forest data summaries have become increasingly
popular.
The \iforest{} algorithm can be roughly stated as:
(i) sample a feature $u$ uniformly at random,
(ii) Along $u$ sample a cut location $c$ uniformly at random \& 
recurse either side of $c$ until the tree has reached maximum height.
Anomalies are then declared based upon their average depth over the forest,
under the expectation that points far from the expected behaviour are easier to identify so are `isolated' more easily in the tree and 
have small average depth.
Cuts exist on the \emph{entire (sub-)domain} over which they are defined.
That is, any cut is continued until it intersects another cut or a boundary,
similar to the Mondrian Process.\footnote{The analogy is not perfect: 
in Mondrian Processes, features to cut are chosen proportional to 
their length.}

However, it was noticed that uniformly sampling features in
\iforest{} could perform suboptimally:
the \rrcf{} rectifies this by
sampling features to cut according to their length \cite{guha2016robust}.
Cuts are restricted to the (sub-)regions of space where data is observed
(just as in the Mondrian Tree)
which enables efficient dynamic changes of the tree as data is added or 
removed.
Given a tree $T$ sampled over data $\mX$, these modifications ensure that the alteration of $T$ to 
$T'$ by adding or removing $\vu$ has the same distribution as sampling 
$T$ on $\mX \cup \{\vu\}$ or $\mX \setminus \vu$, respectively 
(Lemmas 4 and 6 of \cite{guha2016robust}).
The scoring method is related to the expected change in depth of a node 
were a
point (or group of points) not observed; the intuition being that anomalous points cause a
significant change in structure when ignored.
The \rrcf{} also acts as a distance-preserving sketch in the $\ell_1$ norm, suggesting that this data structure is more general than 
its common use-case for anomaly detection.
Interestingly, an extension of Mondrian Forests appears to 
exploit similar properties for estimating the Laplacian Kernel 
\cite{BalLakGhaetal16}.

Finally, the Partial
Identification Forest (\pid{}) is a $k$-ary tree for $k \ge 2$.
In contrast to the previous two approaches, 
the splits are \emph{optimised} 
deterministically over a uniform subsample of the input data to maximise the variance between \emph{sparsity} across subgroups on a feature.
The sparsity of a set of points is roughly the volume of the point set normalised by the volume of the region enclosed by a cut.
It could be problematic to adapt the cuts for removed or new datapoints, so the \pid{} may not be ideal for heterogeneous data streams.

\section{Anomaly Detection Experiments} 
\label{sec: anomaly-detection}
\textbf{Datasets.}
We test on all datasets from the open data repository in 
the Python Outlier Detection library (\acrshort{pyod}) 
(\cite{zhao2019pyod},\cite{Rayana:2016}) \& selected streaming datasets
from the Numenta Anomaly Benchmark repository 
\cite{ahmad2017unsupervised}, \cite{NAB}.
The data are summarised in \Cref{table: data-details}, 
\Cref{sec: anom-appendix}, ranging over
$n \approx 10^2 - O(10^5)$ \& $D \approx 10 - 400$.
A mixture of batch \& streaming 
data are present, as well as data containing continuous \& categorical variables.
The prevalence of anomalies ranges from $0.03\%$ to $36\%$.\footnote{We remark that $36\%$ seems unusually high
for anomaly detection, but 
follow the conventions 
from \cite{zhao2019pyod}}
For stability in the volume computations MinMax feature scaling
into $[0,1]$ was performed for $D \ge 50$. 
Univariate streaming datasets were transformed into 10 dimensions
by applying the common `shingling' technique of combining 10 
consecutive points into one feature vector.
As in \cite{gopalan2019pidforest}, our performance metric
is the \emph{area under curve} (\acrshort{auc}) for the \emph{receiver operating charactersitic} (ROC) curve.

\textbf{Our Approach.}
We sample a forest $F = \cup_i T_i$ of 100 trees on the data $\mX$ of 
batch or streaming \MPT{s} (\acrshort{bmpf} \& \acrshort{smpf}).
Both \acrshort{bmpf} \& \acrshort{mpf} should have 
a lifetime parameter (weighted depth) to govern
the length of the trees but competing methods are more traditional 
$k$-ary trees so we choose $\lambda = \infty$
(as in \cite{lakshminarayanan2014mondrian}) and set a max 
absolute depth of 10 for consistency with \cite{gopalan2019pidforest}.

\textbf{Competing Methods.}
We test against the following random forest algorithms for anomaly 
detection: \emph{Isolation Forest} (\iforest), \emph{Robust Random
Cut Forest} (\rrcf) and \emph{\pid}.
For both \rrcf{} \& \pid{} we utilised opensource implementations
available at \cite{bartos_2019_rrcf} \& \cite{sharan:2019}; 
all other methods are implemented in scikit-learn \cite{scikit-learn}.
For the most meaningful comparison with 
\cite{gopalan2019pidforest}, we adopt exactly their experimental
methodology using default parameters for all scikit-learn methods,
a forest of size 500 with at most 256 points for \rrcf, and 50 trees of depth 10 over a uniform 
sample of 100 points for \pid.
Results for non-random forest approaches 
(e.g. \acrshort{knn},\acrshort{pca}) are in \Cref{tab:non-rf-baseline}, \Cref{sec: anom-appendix}.

\textbf{Performance Summary.}
The batch methods (\acrshort{bmpf}, \iforest{}, \pid{}) all generate 
static data structures.
Although the internal parameters can be incremented on observing data,
the structures do not easily adapt to streaming data.
The two streaming methods (\acrshort{smpf}, \rrcf{}) are adaptive  
structures which can be easily maintained on observing new data.
We compare the batch methods and streaming methods separately: our 
results are summarised in \Cref{tab:mean_ranks} which shows that 
our batch and streaming solutions perform comparably to prior state of
the art.
The full \acrshort{auc} results over the entire \acrshort{pyod}
repository are given in \Cref{tab:pidforest-baseline} which subsumes
the previous benchmark in \cite{gopalan2019pidforest}.
Note that we have not optimised parameter choices for performance,
indicating that the parameter settings for \acrshort{bmpf}, 
\acrshort{smpf} are good defaults - an important feature for anomaly
detection.
An advantage of both \acrshort{bmpf} \& \acrshort{smpf} is that they 
both use the same underlying data structures as \iforest{} \& 
\rrcf{} while adding  additional lightweight probabilistic structure
relying only quantities that can be computed easily from the stored
parameters at every node (e.g volumes).


\begin{table}[htbp]
\caption{Mean Rank and Num. wins for all methods. The batch methods are tested separately from 
the streaming methods (\acrshort{smpf}, \rrcf).}
\label{tab:mean_ranks}
\centering
\begin{tabular}{@{}llll|ll@{}}
\toprule
          & \acrshort{bmpf} & iForest & PidForest & \acrshort{smpf} & RRCF \\ \midrule
Mean Rank & 1.94 & 1.77     & 2.29                & 1.47 & 1.53 \\
Num. Wins & 18   & 24       & 16                  & 32   & 28   \\ \bottomrule
\end{tabular}
\end{table}

\paragraph{Conclusion.}
We have introduced the random forest consisting of \MPT{}s.
These trees have natural interpretations as density estimators of the underlying
distribution of data.
Our approach relates open questions concerning anomaly detection in \cite{lakshminarayanan2016decision} through the lens of density
estimation, thus resolving the open question in \cite{balog2015mondrian}.
Our method enables interpretable anomaly detection as we can threshold in the probability
domain and use masses rather than densities.

In addition, our random forest can be maintained on a dynamic data stream with
insertions and deletions, thus allowing the scalability required for large-data.
In future work, we plan a more in-depth analysis of the performance on data
streams and a rigorous study of the \MPT{} as a density estimator and change-point
detector, rather than simply an anomaly detector.

Finally, there are several directions in which this work could be
extended to allow scalability to higher dimensions by applying random
rotations and/or projections after cuts.
This has the effect of introducing oblique cuts into the space as opposed to
axis-aligned cuts, and could be of further benefit.
Another area for investigation would be to study the effect of \emph{approximate
counting} for the \PTree{} parameters using sketches such as,
for example, the \cm{} sketch.

\newpage

\begin{ack}
  CD is supported by European Research Council grant ERC-2014-CoG 647557.
  We thank Shuai Tang for helpful discussions concerning the experiments and 
  maunscript preparation.
\end{ack}

\section*{Broader Impact}

Important applications of anomaly detection include cybersecurity intrusion detection, operational metrics monitoring, IoT signals (such as detecting broken sensors), and fraud detection. 
Thus, our contribution can have impact across all these domains. 
While there are many applications of varying ethical value that use anomaly detection, such as the possibility for misuse by a surveillance state to "detect anomalous citizen behavior", we believe that by focusing on the addition of interpretability to this solution helps to mitigate the misuses possible and allow better auditing of systems that do make use of anomaly detection \cite{Borradaile2020}.

One application in this domain where there are fairness concerns is regarding the rate of "anomalies" triggered by certain subgroups in fraud detection. 
A poorly calibrated or heuristic measure of anomalous behavior in this setting has the potential to discriminate against subgroups, where the data may be more sparse and thus more likely to appear anomalous.
In this case, additional interpretability of how the model chooses anomalies is extremely important, as it allows the system operator to properly calibrate, using existing probabilistic fairness techniques, to remove or otherwise mitigate discrimination \cite{Davidson2019AFF}.

We present a method that enhances the state of the art for streaming anomaly detection by casting the problem as one of probabilistic density estimation.
Modeling the problem in this way brings the immediate benefit of interpretability in the anomaly space:
typical approaches such as thresholding at say 3 standard deviations away from the mean or median is a standard way of declaring outliers
in applications
but may not be suitable in settings when arbitrary scoring metrics are
proposed
Importantly however, the reframing of this into probability space allows future work to integrate other important socio-technical properties such as privacy and fairness into the same solution, for which there is much research in the field.

Developing accurate, efficient methods for dealing with or summarizing streaming data has the potential to reduce environmental impact significantly, as summarized data is less expensive to send and dealing with data in a localized manner (i.e. on device) removes the need to send data into the cloud for further computation.
This enhancement of downstream analytics also inherently allows for more privacy, by aggregating less raw data together in the cloud. Additional research into how streaming summary methods can be applied in such cases is an exciting area in the preservation of user privacy. Privacy, differential privacy in particular, in the regime of anomaly detection involves a trade-off between knowing enough about a particular data point to determine its anomaly status and the plausible deniability of that data-point. 
Improving the capabilities of private, useful models for anomaly 
detection could be an important area for future work;
for example, integrating existing differential privacy models for kd-trees \cite{cormode2012differentially} with the interpretable
anomaly detectors we have proposed.

\newpage

\bibliographystyle{plain} 
\bibliography{neurips_2020}

\begin{thebibliography}{10}

\bibitem{aggarwal2015outlier}
Charu~C Aggarwal.
\newblock Outlier analysis.
\newblock In {\em Data mining}, pages 237--263. Springer, 2015.

\bibitem{ahmad2017unsupervised}
Subutai Ahmad, Alexander Lavin, Scott Purdy, and Zuha Agha.
\newblock Unsupervised real-time anomaly detection for streaming data.
\newblock {\em Neurocomputing}, 262:134--147, 2017.

\bibitem{outlier-sklearn}
Albert~Thomas Alexandre~Gramfort.
\newblock Comparing anomaly detection algorithms for outlier detection on toy
  datasets.
\newblock
  \url{https://scikit-learn.org/stable/auto_examples/miscellaneous/plot_anomaly_comparison.html}.

\bibitem{angiulli2002fast}
Fabrizio Angiulli and Clara Pizzuti.
\newblock Fast outlier detection in high dimensional spaces.
\newblock In {\em European Conference on Principles of Data Mining and
  Knowledge Discovery}, pages 15--27. Springer, 2002.

\bibitem{NAB}
Numenta authors.
\newblock Numenta anomaly benchmark.
\newblock \url{https://github.com/numenta/NAB/tree/master/data}.

\bibitem{BalLakGhaetal16}
Matej Balog, Balaji Lakshminarayanan, Zoubin Ghahramani, Daniel~M. Roy, and
  Yee~Whye Teh.
\newblock The {M}ondrian kernel.
\newblock In {\em 32nd Conference on Uncertainty in Artificial Intelligence
  (UAI)}, June 2016.

\bibitem{balog2015mondrian}
Matej Balog and Yee~Whye Teh.
\newblock The {M}ondrian process for machine learning.
\newblock {\em arXiv preprint arXiv:1507.05181}, 2015.

\bibitem{bartos_2019_rrcf}
Matthew Bartos, Abhiram Mullapudi, and Sara Troutman.
\newblock {rrcf: Implementation of the Robust Random Cut Forest algorithm for
  anomaly detection on streams}.
\newblock {\em {The Journal of Open Source Software}}, 4(35):1336, 2019.

\bibitem{Borradaile2020}
Glencora Borradaile, Brett Burkhardt, and Alexandria LeClerc.
\newblock Whose tweets are surveilled for the police: An audit of a
  social-media monitoring tool via log files.
\newblock In {\em Proceedings of the 2020 Conference on Fairness,
  Accountability, and Transparency}, FAT* ’20, page 570–580, New York, NY,
  USA, 2020. Association for Computing Machinery.

\bibitem{bottou2007support}
L{\'e}on Bottou and Chih-Jen Lin.
\newblock Support vector machine solvers.
\newblock {\em Large scale kernel machines}, 3(1):301--320, 2007.

\bibitem{breunig2000lof}
Markus~M Breunig, Hans-Peter Kriegel, Raymond~T Ng, and J{\"o}rg Sander.
\newblock Lof: identifying density-based local outliers.
\newblock In {\em Proceedings of the 2000 ACM SIGMOD international conference
  on Management of data}, pages 93--104, 2000.

\bibitem{cormode2012differentially}
Graham Cormode, Cecilia Procopiuc, Divesh Srivastava, Entong Shen, and Ting Yu.
\newblock Differentially private spatial decompositions.
\newblock In {\em 2012 IEEE 28th International Conference on Data Engineering},
  pages 20--31. IEEE, 2012.

\bibitem{Davidson2019AFF}
Ian Davidson and Selvan~Suntiha Ravi.
\newblock A framework for determining the fairness of outlier detection.
\newblock In {\em Proceedings of the 24th European Conference on Artificial
  Intelligence (ECAI2020)}, 2029.

\bibitem{diethe2019continual}
Tom Diethe, Tom Borchert, Eno Thereska, Borja de~Balle Pigem, and Neil
  Lawrence.
\newblock Continual learning in practice.
\newblock In {\em NeurIPS 2018 Workshop on Continual Learning}, 2018.

\bibitem{flach2015precision}
Peter Flach and Meelis Kull.
\newblock Precision-recall-gain curves: Pr analysis done right.
\newblock In {\em Advances in neural information processing systems}, pages
  838--846, 2015.

\bibitem{gopalan2019pidforest}
Parikshit Gopalan, Vatsal Sharan, and Udi Wieder.
\newblock Pidforest: anomaly detection via partial identification.
\newblock In {\em Advances in Neural Information Processing Systems}, pages
  15783--15793, 2019.

\bibitem{guha2016robust}
Sudipto Guha, Nina Mishra, Gourav Roy, and Okke Schrijvers.
\newblock Robust random cut forest based anomaly detection on streams.
\newblock In {\em International conference on machine learning}, pages
  2712--2721, 2016.

\bibitem{lakshminarayanan2016decision}
Balaji Lakshminarayanan.
\newblock {\em Decision trees and forests: a probabilistic perspective}.
\newblock PhD thesis, UCL (University College London), 2016.

\bibitem{lakshminarayanan2014mondrian}
Balaji Lakshminarayanan, Daniel~M Roy, and Yee~Whye Teh.
\newblock Mondrian forests: Efficient online random forests.
\newblock In {\em Advances in neural information processing systems}, pages
  3140--3148, 2014.

\bibitem{liu2008isolation}
Fei~Tony Liu, Kai~Ming Ting, and Zhi-Hua Zhou.
\newblock Isolation forest.
\newblock In {\em 2008 Eighth IEEE International Conference on Data Mining},
  pages 413--422. IEEE, 2008.

\bibitem{prg}
Miquel Perello~Nieto Meelis~Kull, Telmo de Menezes e Silva~Filho.
\newblock pyprg: Python package for creating precision-recall-gain curves and
  calculating area under the curve.
\newblock \url{https://github.com/meeliskull/prg/tree/master/Python_package}.

\bibitem{muller2013polya}
Peter M{\"u}ller, Abel Rodriguez, et~al.
\newblock P{\'o}lya trees.
\newblock In {\em Nonparametric Bayesian Inference}, pages 43--51. IMS and ASA,
  2013.

\bibitem{scikit-learn}
F.~Pedregosa, G.~Varoquaux, A.~Gramfort, V.~Michel, B.~Thirion, O.~Grisel,
  M.~Blondel, P.~Prettenhofer, R.~Weiss, V.~Dubourg, J.~Vanderplas, A.~Passos,
  D.~Cournapeau, M.~Brucher, M.~Perrot, and E.~Duchesnay.
\newblock {Scikit-learn: Machine Learning in Python }.
\newblock {\em Journal of Machine Learning Research}, 12:2825--2830, 2011.

\bibitem{ramaswamy2000efficient}
Sridhar Ramaswamy, Rajeev Rastogi, and Kyuseok Shim.
\newblock Efficient algorithms for mining outliers from large data sets.
\newblock In {\em Proceedings of the 2000 ACM SIGMOD international conference
  on Management of data}, pages 427--438, 2000.

\bibitem{Rayana:2016}
Shebuti Rayana.
\newblock {ODDS} library.
\newblock \url{http://odds.cs.stonybrook.edu}, 2016.

\bibitem{roy2008mondrian}
Daniel~M Roy, Yee~Whye Teh, et~al.
\newblock The {M}ondrian process.
\newblock In {\em NIPS}, pages 1377--1384, 2008.

\bibitem{scholkopf2001estimating}
Bernhard Sch{\"o}lkopf, John~C Platt, John Shawe-Taylor, Alex~J Smola, and
  Robert~C Williamson.
\newblock Estimating the support of a high-dimensional distribution.
\newblock {\em Neural computation}, 13(7):1443--1471, 2001.

\bibitem{sharan:2019}
Vatsal Sharan.
\newblock {PIDF}orest library.
\newblock \url{https://github.com/vatsalsharan/pidforest}, 2019.

\bibitem{shyu2003novel}
Mei-Ling Shyu, Shu-Ching Chen, Kanoksri Sarinnapakorn, and LiWu Chang.
\newblock A novel anomaly detection scheme based on principal component
  classifier.
\newblock Technical report, Miami University Department of Electrical and
  Computer engineering, 2003.

\bibitem{OpenML2013}
Joaquin Vanschoren, Jan~N. van Rijn, Bernd Bischl, and Luis Torgo.
\newblock {OpenML}: Networked science in machine learning.
\newblock {\em SIGKDD Explorations}, 15(2):49--60, 2013.

\bibitem{zhao2019pyod}
Yue Zhao, Zain Nasrullah, and Zheng Li.
\newblock {PyOD}: A {P}ython toolbox for scalable outlier detection.
\newblock {\em Journal of Machine Learning Research}, 20(96):1--7, 2019.

\end{thebibliography}
\newpage

\appendix
\section{Sampling Mondrian Trees, \PTree{s} and \MPT{s}}

\begin{table}[htbp]
\caption{Lay summary of the various methods explored in 
our work.}
\label{tab:lay-summary}
\centering
\begin{tabular}{@{}ll@{}}
\toprule
Phrase              & Lay Summary           \\ \midrule
\multicolumn{2}{c}{Bayesian Nonparametrics} \\ \midrule
Mondrian Process &
  \begin{tabular}[c]{@{}l@{}}
      A binary tree which partitions sequentially partitions a space by cutting features.\\ 
      Cut dimensions chosen with probability proportional to length.\\ 
      Every node has a time.\\ 
      These trees are parametrised by the lifetime $\lambda > 0$ but may be infinite if\\
      the sum of cutting times repeatedly lies slightly less than $\lambda$
  \end{tabular} \\ \midrule
Mondrian Tree &
  \begin{tabular}[c]{@{}l@{}}
      A binary tree which partitions an input space by cutting features.\\ 
      Cut dimensions are chosen with probability proportional to length.\\ 
      All nodes have times.\\ 
      Tree is guaranteed to be finite.
  \end{tabular} \\ \midrule
P{\'o}lya Tree &
  \begin{tabular}[c]{@{}l@{}}
      A probabilistic structure which takes as input any binary partition and distributes \\ 
      mass throughout the partition.\\ 
      Partition can be finite or infinite
  \end{tabular} \\ \midrule
\multicolumn{2}{c}{Anomaly Detectors}       \\ \midrule
iForest             &    
 \begin{tabular}[c]{@{}l@{}}
      Random forest generated by subsampling the dataset \& building a binary tree. \\
      Feature $u$ sampled uniformly followed by a uniformly sampled cut in $u$.\\
      The scoring mechanism is average depth: anomalous points are easy to isolate \\
      so will have a low average depth compared to normal points.
  \end{tabular} \\ \midrule 
RRCF                & 
 \begin{tabular}[c]{@{}l@{}}
      Random forest generated by subsampling the dataset \& building a binary tree. \\
      Feature $u$ sampled with probability proportional to length.\\
      Cut location uniformly sampled cut in $u$.\\
      The scoring mechanism is ``codisplacement'': the expected change in \\
      the structure of a tree were a group of points not observed.
  \end{tabular} \\ \midrule 
PidForest          &    
 \begin{tabular}[c]{@{}l@{}}
      Random forest generated by subsampling the dataset \& building a $k$-ary tree \\
      for some $k$ to be chosen. \\ 
      Features and cut locations chosen deterministically.\\
      The scoring mechanism is ``sparsity'': roughly the volume of a pointset \\
      divided by the volume of the region enclosing it.
  \end{tabular} \\ \bottomrule 
\end{tabular}
\end{table}

For clarity we describe the structures necessary to introduce our \MPT{s} which 
are summarised in \Cref{tab:lay-summary}.\footnote{
Please note that between paper submission and supplementary submission
we added \Cref{tab:lay-summary} so the table indexing has been incremented
by 1 from the paper originally submitted.}
We will begin with the \textbf{Mondrian Process} which can be succinctly described as:
given an input domain $\gD$ and a lifetime $\lambda > 0$, choose a direction (feature)
to cut with probability proportional to length.
Next, choose a cut location uniformly at random on the selected feature and 
split into two sets less than and greater than the cut location.
This cut procedure has a random cost associated to the ``linear dimension''
(sum of the lengths) of the region at a given node and the process is 
repeated until the lifetime is exhausted by accumulating the random costs.
An implementation is given in \cite{roy2008mondrian}.

The \textbf{Mondrian Tree} builds on the Mondrian Process by building the trees in a more
data-aware fashion.
At a high-level this process is similar to the Mondrian Process except every cut
takes place on a restriction of space to the bounding box on which observations 
are made.
The advantage of this is that cuts are guaranteed to pass through observations
which in high dimensions could result in substantially shortened trees.
Mondrian Trees can also be sampled online which makes them highly efficient.
However, the price to pay for these efficiency gains is that behaviour outside of the 
bounding boxes cannot be modelled.

While the previous two methods are useful for partitioning the data into clusters,
they make no statements about the underlying density of the dataset.
To accommodate this we introduce  the \textbf{\PTree{}} which is a Bayesian nonparametric
model for estimating the underlying density function generating the data.
The \PTree{} model takes as input a binary nested partition of an input space $\gD$,
(represented by a binary tree) and assigns probability to each of the bins (nodes in 
the tree).
Given a point in a bin indexed $A_{b(j)}$, the presence of a point in the bins 
$A_{b(j)1}$ is modelled by a Bernoulli distribution with parameter $p$.
Let $d_j$ denote the depth of $A_{b(j)}$ and 
$V_0, V_1$ denote the volumes of the the bins $A_{b(j)0}, A_{b(j)1}$, 
respectively.
The prior distribution for $p$ is a Beta distribution which has parameters:
\begin{align}
    \alpha_{j0} &= \gamma \left(d_j + 1\right)^2 \frac{V_0}{V_0 + V_1} \\
    \alpha_{j1} &= \gamma \left(d_j + 1\right)^2 \frac{V_1}{V_0 + V_1}
\end{align}
for a hyperparameter $\gamma > 0$ denoting the strength of the prior distribution.
The posterior parameters for the $\alpha_{jk}$ are then incremented by the 
number of points observed in the $A_{b(j)k}$ bin for $k=0,1$.
An implementation is given in \Cref{alg: polya-tree-sampling} which takes 
as input the partition of space $\gD$, thus requiring an extra pass through the tree.
However, for our applications as defined in \Cref{sec: mpf-anomaly}, 
we will be able to implement this in an online fashion.

Our \textbf{\MPT{}} can be implemented in either a batch or streaming fashion.
For a batch computation, we can adapt the Mondrian Process and easily combine this 
with the \PTree{}.
However, for streaming computation, the `empty space' caused by restricting to 
bounding boxes in the Mondrian Tree procedure is highly problematic and this 
motivated our revised construction, the \acrshort{mpt} as described in 
\Cref{sec: mpt-intro}.
We describe this revision in \Cref{alg: mpt} while the parameter 
update algorithms are presented in \Cref{alg: mpt-parameters}.

\textbf{Generating \MPT{s}: Computational Complexity.}
Combining the \PTree{} with either the Mondrian Process or Mondrian Tree 
incurs only a mild overhead in both time and space as all that needs to be stored is an
extra set of parameters.
For the \textbf{batch \MPT{}} (\Cref{sec: mondrian-polya-forest}) this is simply
3 counters per node ($\alpha_{j0},\alpha_{j1}, P_j$).
In \Cref{sec: mpt-intro} we showed that a two-stage split was necessary for the 
\textbf{streaming \MPT{}} and this 
slightly increases the number of parameters to at most 7 per node (see \ref{sec: mpt-params})
which come from the 2 cut parameters, at most 4 restriction parameters, and the mass 
float $P_j$.
Overall, both methods need $O(|\tree|)$ extra space which, nevertheless, is only a
constant factor more space than is required to build the partitioning tree.

The time cost to evaluate these parameters is $O(d | \tree |)$ as computing the volume 
of every node costs $O(d)$.
Since we make the distinction between \textbf{type I/II observation \& complementary
leaves}, volume comparisons are made over nonzero $D$-dimensional hypervolumes.
This permits the following distinctions
at every node to avoid incurring complex volume computations of the complementary regions.

\textbf{Volume Computation for \acrshort{mpt}.}
Recall that for a node $j$ we sample a cut dimension $\cdim_j$ and in that dimension 
a cut location $\xi_j$.
The node $j$ contains the restriction to bounding box $B_j$ which is split into two 
regions $R_{j0}$ and $R_{j1}$ either side of $\xi_j$.
Node $j$ has volume $V_{B_j} = \vol{j}$
and let $h_j$ denote the length of the sampled dimension $\cdim_j$;
the volumes associated 
with $R_{j0}$ and $R_{j1}$ are:
\begin{align}
    V_{R_{j0}} &= \frac{V_{B_j}}{h_j} |\min_{\vx \in j} \vx_{\cdim_j} - \xi_j | \\
    V_{R_{j1}} &= \frac{V_{B_j}}{h_j}|\max_{\vx \in j} \vx_{\cdim_j} - \xi_j |.
\end{align}


We obtain the volume of the observed region when 
 computing $\Restrict{j}$ for the restriction to bounding boxes either side of 
 the cut $\xi_j$ at $j$.
 Recall that 
 $V_{\lChild{j}} = \vol{ B_{\lChild{j}} }$, so the subtraction 
 $V_{R_{j0}} - V_{\lChild{j}} = V_{B_{j0}^C}$ 
 yields the complementary volume necessary for setting the restriction
 \Polya{} parameters $\rho_{\cdot \in}, \rho_{\cdot \neg}$.
All volumes being supported on $D$-dimensional boxes ensures that none of these
quantities trivially collapse to zero.
If one of the feature lengths is zero then we simply treat such a node as a Type I
 observation leaf.


\begin{algorithm}[htb]
\KwIn{Training data $\mX \in \R^{n \times D}$, lifetime $\lambda > 0$ }
\SetKwFunction{FMain}{SampleMondrianTree}
\SetKwFunction{FSub}{SampleMondrianBlock}
\SetKwProg{Fn}{Function}{:}{}
\Fn{\FMain{$\mX, \lambda$}}{
Initialise $\tree = \emptyset, \leaves{\tree} = \emptyset,
\CutDims = \emptyset, \CutLocs = \emptyset, \CutTimes = \emptyset,
N(\Troot) = \{1,2,\dots,n\}.$\\
    \FSub{$\Troot, \mX_{N(\Troot)}, \lambda$}  \\
}
\SetKwProg{Pn}{Function}{:}{\KwRet}
\Pn{\FSub{$j, \mX_{N(j)}, \lambda$}}{
      $\tree \leftarrow \tree \union \{ j \} $ \\
      For all $d \in [D]$ set
      $\vl_{jd}^{\mX} = \min_d \mX_{N(j)},\vu_{jd}^{\mX} = \max_d \mX_{N(j)}$
      to be the dimension-wise minima and maxima of the observations in $j$ \\
      Let $L = \sum_d (u_{jd}^{\mX} - l_{jd}^{\mX})$ denote the
      \emph{linear dimension} of the data in $j$ \\
      Sample $E \sim \ExpDist{L}$ \\
      \uIf{$\tau_{\parent{j}} + E < \lambda$}{
          Set $\tau_j = \tau_{\parent{j}} + E$ \\
          Sample cut dimension $\cdim_j$ with probability proportional to
          $u_{jd}^{\mX} -  l_{jd}^{\mX}$ \\
          Sample cut location uniformly on the interval
          $[ l_{j \cdim_j }^{\mX}, u_{j \cdim_j}^{\mX}]$ \\

          Set $N(\lChild{j}) = \{ n \in N(j) : X_{n \cdim_j} \le \xi_j \}$ and
          $N(\rChild{j}) = \{ n \in N(j) : X_{n \cdim_j} > \xi_j \}$ \\

          \FSub{$\lChild{j}, \mX_{N(\lChild{j})}, \lambda$}\\
          \FSub{$\rChild{j}, \mX_{N(\rChild{j})}, \lambda$}
      }
      \Else{
      $\tau_j \leftarrow \lambda$ and
      $\leaves{\tree} \leftarrow \leaves{\tree} \cup \{j\}$
  }
}
 \caption{Mondrian Forest Sampling \cite{lakshminarayanan2014mondrian}}
 \label{alg: mondrian-forest}
\end{algorithm}

\begin{algorithm}[htb]
\KwIn{Input domain $\gD \subset \R^{D}$,
      a decision tree $T$ which partitions $\gD$,
      hyperparameter $\gamma > 0$}
\KwOut{Probability distribution $\gP = (P_l)_{l \in \leaves{T}}$}
\SetKwFunction{FMain}{Sample\Polya{}Tree}
\SetKwFunction{FSub}{Update\Polya{}Parameters}
\SetKwProg{Fn}{Function}{:}{}
\Fn{\FMain{$\gD, T, \gamma$}}{
$\Troot = \text{root}\left( T \right)$ \\
$P_{\Troot} = 1$ \Comment{Assume all mass is located in the region $\gD$}\\
 \FSub{$\Troot, \gamma$} \\
 }

 \SetKwProg{Pn}{Function}{:}{\KwRet}
 \Pn{\FSub{$j, \gamma$}}{
    \uIf{$j \in \leaves{T}$}{
    $\text{ProbDensity}(j) = P_j / \vol{j}$
    \Comment{$P_j$ was defined at the preceeding level.}
    }
    \uElse{
    $V_j = \vol{j}$ \\
    Let $R_j = [l_1,u_1] \times \dots \times [l_d,u_d]$ define the region in
    as a product of intervals from the minimum in dimension $i$, $l_i$, to
    the maximum in dimension $i$, $u_i$. \\
    $L_j = \sum_{j=1}^d (u_j - l_j)$ is the linear dimension of the region.\\
    $V_0 = \frac{V_j}{u_{\cdim_j} - l_{\cdim_j}} \cdot | l_{\cdim_j} - \xi_j |$
    \Comment{Volume less than cut $\xi_j$}\\
    $V_1 = \frac{V_j}{u_{\cdim_j} - l_{\cdim_j}} \cdot | u_{\cdim_j} - \xi_j |$
    \Comment{Volume greater than cut $\xi_j$}\\
    $n_0 = \text{NumberOfPoints}(\lChild{j}),
    n_1 = \text{NumberOfPoints}(\rChild{j})$
    \Comment{Number of points in children nodes}\\
    $\alpha_0 = \gamma (d+1)^2 \frac{V_0}{V_0 + V_1} + n_0,
    \alpha_1 = \gamma (d+1)^2 \frac{V_1}{V_0 + V_1} + n_1$
    \Comment{Set prior parameters using \PTree~ and then increment using
    Beta-Binomial conjugacy}\\
    $\mu_j = \frac{\alpha_0}{\alpha_0 + \alpha_1}$
    \Comment{$\E(\BetaDist{\alpha_0}{\alpha_1})$}\\
    $P_{\lChild{j}} = \mu_j P_j,  P_{\rChild{j}} = (1 - \mu_j) P_j$\\
    \FSub{$\lChild{j}, \gamma$}, \FSub{$\rChild{j}, \gamma$}
    }
}
 \caption{\PTree{} Sampling. Sets probability (density) for all nodes in the
 given random partition $T$.}
 \label{alg: polya-tree-sampling}
\end{algorithm}

\begin{algorithm}[p]
 \KwIn{Training data $\mX \in \R^{n \times D}$, at least one of lifetime
 $\lambda > 0$ or maximum tree height $m$, \PTree{} hyperparameter $\gamma > 0$}
 \KwOut{A classical Mondrian Tree data structure $T$;
 Partition $\Pi$ over $T$ such that
 $\gP = (P_l)_{l \in \leaves{T}}$
 is a probability distribution over the leaves of $T$ induced by \PTree{} prior}
 \SetKwFunction{FMain}{SampleMondrian\Polya{}Tree}
 \SetKwFunction{FSub}{SampleMondrian\Polya{}Block}
 \SetKwFunction{FSubCut}{SetCutParameters}
 \SetKwFunction{FSubObs}{SetRestrictionParameters}
\SetKwFunction{FRestrict}{Restrict}
\SetKwFunction{FCut}{Cut}
 \SetKwProg{Fn}{Function}{:}{}
 \Fn{\FMain{$\mX, \lambda$}}{
 Initialise $\tree = \emptyset, \leaves{\tree} = \emptyset,
 \CutDims = \emptyset, \CutLocs = \emptyset, \CutTimes = \emptyset,
 N(\Troot) = \{1,2,\dots,n\}.$\\
 $\Troot.\text{ObservedVolume}  = \vol{\bbox{\mX_{N(\Troot)}}}$ \\
     \FSub{$\Troot, \mX_{N(\Troot)}, \lambda$}  \\
 }
 \SetKwProg{Pn}{Function}{:}{\KwRet}
 \Pn{\FSub{$j, \mX_{N(j)}, \lambda$}}{
       $\tree \leftarrow \tree \union \{ j \} $ \\
       $B_j \leftarrow \bbox{ \mX_{N(j)}}, L = \LinDim{B_j}$
       \Comment{\emph{linear dimension} of the bounding box for $j$ }\\
       Sample $E \sim \ExpDist{L}$ \\
       \uIf{$\tau_{\parent{j}} + E < \lambda \text{~and all feature lengths are positive}$}{
           Set $\tau_j = \tau_{\parent{j}} + E$ \\
           \FCut{$j, \mX_{N(j)}, B_j$} \\

           Set $N(\lChild{j}) = \{ n \in N(j) : \mX_{n \cdim_j} \le \xi_j \}$ and
           $N(\rChild{j}) = \{ n \in N(j) : \mX_{n \cdim_j} > \xi_j \}$ \\

           \FRestrict{$j, N(\lChild{j})$} \\
           \FRestrict{$j, N(\rChild{j})$}

           \FSub{$\lChild{j}, \mX_{N(\lChild{j})}, \lambda$}\\
           \FSub{$\rChild{j}, \mX_{N(\rChild{j})}, \lambda$}
       }
       \Else{
            \uIf{any feature length is 0}{
                $\tree \leftarrow \tree \cup \{ j \}$
                \Comment{Bounding box supported on $ < d$ dimensions: Type I Observed leaf}
            }
            \Else{
             \FRestrict{$j, N(j)$}
             \Comment{Restrict once more to generate a \emph{complementary leaf}} \\
            }
       $\tau_j \leftarrow \lambda$ and
       $\leaves{\tree} \leftarrow \leaves{\tree} \cup \{j\}$
       \Comment{Type II Observed leaf (see \Cref{sec: mpt-intro})}
   }
 }

\SetKwProg{Pn}{Function}{:}{\KwRet}
\Pn{\FCut{$j, \mX_{N(j)}, B_j$}}{
    Sample cut dimension $\cdim_j$ with probability proportional to
    $u_{jd}^{\mX} -  l_{jd}^{\mX}$ \\
    Sample cut location $\xi_j$ uniformly on the interval
    $[ l_{j \cdim_j }^{\mX}, u_{j \cdim_j}^{\mX}]$ \\
    ${R_{\TreeLeft}} = \{\vz \in B_j : \vz_{\cdim_j}  \le \xi_j\}$ \\
    ${R_{\TreeRight}} = \{\vz \in B_j : \vz_{\cdim_j}  > \xi_j\}$ \\
    $n_{\TreeLeft} = |\mX_{N(j)} \cap {R_{\TreeLeft}}|$ \\
    $n_{\TreeRight} = |\mX_{N(j)} \cap {R_{\TreeRight}}|$ \\
    $V_{\TreeLeft} = \vol{{R_{\TreeLeft}}},
    V_{\TreeRight} = \vol{{R_{\TreeRight}}}$ \\
    $d_j = \TreeDepth{j}$ \Comment{Absolute depth in Mondrian Tree} \\
    \FSubCut{$2d_j, n_{\TreeLeft}, n_{\TreeRight},V_{\TreeLeft}, V_{\TreeRight}$}
}

\SetKwProg{Pn}{Function}{:}{\KwRet}
\Pn{\FRestrict{$j, \mX_{N(j)}$}}{
    $d = j.\text{depth}$
    \Comment{Get absolute depth in Mondrian tree} \\
    $n_{\textsf{obs}} = |N(j)|$ \Comment{Num. points in node} \\
    $V_p = \parent{j}.\text{ObservedVolume}$
    \Comment{Parent volume} \\
    $V_o = \vol{B(\mX_{N(j)})}$
    \Comment{Observed volume} \\
    $V_c = V_p - V_o$
    \Comment{Complementary volume} \\
    $\rho_0^*, \rho_1^*$ = \FSubObs{$2d+1, n_{\textsf{obs}}, V_o, V_c$}
    \Comment{Set Beta parameters.}\\
    \Return{$\rho_0^*, \rho_1^*$}
}

  \caption{\MPT{} Sampling.  Subroutines:
  \Cref{alg: mpt-parameters}}
  \label{alg: mpt}
 \end{algorithm}
 
 \begin{algorithm}[htb]
 \SetKwFunction{FSubCut}{SetCutParameters}
 \SetKwFunction{FSubObs}{SetRestrictionParameters}
 \SetKwProg{Pn}{Function}{:}{\KwRet}
 \Pn{\FSubCut{$\text{depth}, n_{\TreeLeft}, n_{\TreeRight},
             V_{\TreeLeft}, V_{\TreeRight}$}}{
     $d = \text{depth}$ \\
     $j.\chi_0^* = \gamma (d + 1)^2
     \frac{ V_{\TreeLeft} }{ V_{\TreeLeft} + V_{\TreeRight}} + n_{\TreeLeft},
     j.\chi_1^* = \gamma (d + 1)^2
     \frac{ V_{\TreeRight} }{ V_{\TreeLeft} + V_{\TreeRight}} + n_{\TreeRight}$ \\
 }
 \SetKwProg{Pn}{Function}{:}{\KwRet}
 \Pn{\FSubObs{$\text{depth},\text{NodeSize}, \text{ObservedVolume}, \text{ComplemetaryVolume}$}}{
     $d = \text{depth}; \qquad$ 
     $n = \text{NodeSize}$ \Comment{Number of points in observed bounding box} \\
     $V_{obs} = \text{ObservedVolume}, V_{comp} = \text{ComplemetaryVolume}$ \\
     $\rho_0^* = \gamma (d+1)^2 \frac{V_{obs}}{V_{obs} + V_{comp}} + n$,
     $\rho_1^* = \gamma (d+1)^2 \frac{V_{comp}}{V_{obs} + V_{comp}}$ \\
     \Return{$\rho_0^*, \rho_1^*$}
 }
 \caption{Subroutines for setting Beta Distribution parameters for the
 \MPT.
 Note that the depth parameters in these subroutines refer to depth in \PTree, not absolute
 depth in Mondrian Tree!}
 \label{alg: mpt-parameters}
 \end{algorithm}
\newpage
\section{\acrshort{mpt}: Insertions and Deletions} 
\label{sec: mpt-appendix}

A substantial benefit of the Mondrian Tree construction is that it can 
be built online as new data is seen.
The key idea underpinning this is
\emph{projectivity} (\Cref{lem:projectivity} \Cref{sec: preliminary}), 
which asserts that if a Mondrian tree 
$T \sim \MTree{\mX}{\lambda}$ is sampled and a new point $\vz$ is 
observed, then inserting $\vz$ into $T$ to generate $T'$ yields
$T' \sim \MTree{\mX \cup \vz}{\lambda}$ \cite{lakshminarayanan2014mondrian}; 
moreover, this process is efficient.
This is where the restriction of a cut $\xi_j$ to the 
bounding box $B_j$ is critical, because it permits the sequential
addition of $\vz$ into $T$ while preserving the distribution over which 
$T$ was sampled had $\vz$ been seen prior to sampling $T$!
We adapt the online update procedures from Mondrian Trees
to streaming \MPT{s} by invoking projectivity and then recognising that 
the necessary parameters can be easily incremented as the 
data is observed.
Inserting a point $\vz$ into tree $T$ is denoted 
$T' \sim \MPTreeInsert{T}{\vz}$.

However, for data streams we also need the capability to delete from the tree;
this is where the link with the \rrcf{} work becomes 
necessary, as we can adapt their deletion mechanism for the 
\MPT{} setting.
Our alteration is necessary for the Mondrian Tree setting as
nodes have an associated time which cannot
exceed the lifetime budget $\lambda$ and deleting a point on the bounding box
can affect the times of all nodes in the subtree rooted at that node.
In this setting, the point to delete, $\vz$, is chosen
ahead of time, hence the algorithm is deterministic which is why we will
write $T' = \MPTreeDelete{T}{\vz}$ (in contrast to 
$T' \sim \MPTreeInsert{T}{\vz}$) for deleting $\vz$ from $T$.
The following lemmas summarise the insertion and deletion 
procedures from \cite{lakshminarayanan2014mondrian} and \cite{guha2016robust}
to account for the additional \PTree{} parameters that we need
when using the \MPT.
The insertions procedure is described in \Cref{alg: mpt-insertion}, 
and \Cref{alg: mpt-deletion} illustrates the deletion mechanism.


\begin{lemma}[Insertions] \label{lem:mpt-insertion}
Let $T \sim \MondPolyaTree{\mX}{\lambda}$ be a \MPT sampled over 
data $\mX$ with lifetime $\lambda > 0$.
If $\vz$ is a new observation and $T' \sim \MPTreeInsert{T}{\vz}$ then 
$T' \sim \MondPolyaTree{\mX \cup \vz}{\lambda}$.
\end{lemma}

\begin{proof} 
The tree that we sample and store is exactly a Mondrian Tree, hence we 
invoke projectivity so that $T'$ is a valid Mondrian Tree over 
$\mX \cup \vz$.
Since the Mondrian Tree $T$ implicitly but uniquely defines a \MPT{}
which partitions the input space, projectivity also applies to the 
\MPT{} structure as a random partition.
Additionally, we need to alter the (cut and restrict) Beta parameters for
every node which are affected by the insertion of $\vz$ in tree $T$. 
However, this amounts to simply incrementing counts over the subtree: 
updating the parameters is sufficient as we only need the expected value
of every Beta distribution.

\end{proof}

\begin{lemma}[Deletions] \label{lem:mpt-deletion}
Let $T \sim \MondPolyaTree{\mX}{\lambda}$ and 
let $\vz$ be the point to be removed from $\mX$ and $T$.
If $T' = \MPTreeDelete{T}{\vz}$ then 
$T' \sim \MondPolyaTree{\mX \setminus \vz}{\lambda}$.
\end{lemma}

\begin{proof}

First, locate the deepest node $j$ containing $\vz$, there are two cases:
(i) $\vz$ is \emph{internal} to the box $B_j$
(ii) $\vz$ is a \emph{boundary point} defining part of the bounding box
$B_j$ (i.e. it is maximal or minimal at $j$ in one dimension).
If $\vz$ is internal to $B_j$ then we are free to simply remove it from 
$j$ and decrement the necessary counts.
Otherwise, deleting $\vz$ causes a change to the bounding box: let 
$B_j'$ denote the new bounding box for $j$ under the removal of $\vz$.
Now, it must be the case that $L_j' = \LinDim{B_j'}$ is \emph{at most}
$L_j = \LinDim{B_j}$.
However, if there is a $\vu \ne \vz$ 
in $j$ but is equal to $\vz$ in \emph{all} dimensions on which $\vz$ lies 
on the boundary, then we could treat $\vz$ as an internal point and 
remove then decrement.
So assume $\vz$ uniquely defines $B_j$ in the required dimensions, hence
$L_j' < L_j$ so the exponential distribution used to generate the 
node time $\tau_j$ is different under the absence of $\vz$.
Let $F(t) = \cdf{\ExpDist{L_j}}(t)$ and $G(t) = \cdf{\ExpDist{L_j'}}(t)$ be the
CDF functions of the exponential distributions $\ExpDist{L_j}$ and 
$\ExpDist{L_j'}$, respectively as functions of time $t$.
The mass associated to time $\tau_j$ is $\psi = F(\tau_j)$ hence, the time 
with the same mass in $G(t)$ is
$\tau_j' = G^{-1}(\psi)$ (these are straightforward for the exponential
distribution since $\cdf{\ExpDist{\zeta}}(t) = 1 - \exp(-\zeta t)$).
Finally, since $L' < L$, we must have $\tau_j' > \tau_j$ so the time 
has increased, meaning we must check whether $\tau_j' < \lambda$.
If so, then keep $j$, else contract $j$ and its descendants into $\parent{j}$.
This approach must be done for every node on the path from $\Troot$ to 
 $j$ which contains $\vz$ so in the worst case is 
 $O(d \cdot \TreeDepth{T})$.
 Finally, it remains to decrement all necessary counts which were affected 
 by the presence of $\vz$ on the path from $\Troot$ to $j$ 
 (or the contracted ancestor of $j$).
\end{proof}

 \textbf{Complexity: Insertions \& Deletions}
 Both procedures are efficient and are dominated by the time it takes to locate
the locate the node which stores query point and requires checking inclusion in a
bounding box at $O(D)$ cost a maximum of $\TreeDepth{T}$ times, hence  $O(D \TreeDepth{T})$ overall.
 Note that this is the absolute depth measured in the Mondrian
 Tree sense, not the adjusted depth to account for the \PTree{}
 construction as defined prior to \Cref{eq: mpt-prior},
 nor the lifetime $\lambda$ which could potentially be large.
 Since we only store the Mondrian Tree which generates the \MPT{}
 which, in expectation, should be balanced and hence 
 $\TreeDepth{T} = \Theta(\log n)$.
 In the random forest literature (\cite{liu2008isolation}, 
 \cite{guha2016robust}, \cite{gopalan2019pidforest}), the depth is 
 typically a parameter of small magnitude relative to 
 the size of input data, usually 10.
 Hence, the presence of the maximum tree depth term in the above
 time complexity bounds is not problematic.

\begin{algorithm}[h]
 \KwIn{Mondrian Tree $T = (\tree, \CutDims, \CutLocs, \CutTimes)$}
 \KwOut{Mondrian Tree $T$ sampled over $\mX \cup \vz$}
\SetKwFunction{FMain}{$\textsf{MT}_+$}
\SetKwFunction{FSub}{MTx}
\SetKwProg{Fn}{Function}{:}{}
\Fn{\FMain{$T, \mX, \lambda, \vz$}}{
$\Troot = \textsf{root}(T)$ \\ 
 \FSub{$T, \mX, \lambda, \vz, \Troot$} \\
 }
 \SetKwProg{Pn}{Function}{:}{\KwRet}
 \Pn{\FSub{$T, \mX, \lambda, \vz, j$}}{
    $\ve^l = \max(l_j^{\mX} - \vz,0), \ve^u = \max(\vz - u_j^{\mX},0)$ 
    \Comment{$\ve^l, \ve^u  = \vzero_d$ iff $z \in B_j$}\\
    Increment the \emph{observed restriction} parameter $\rho_0$ by 1 \\ 
    Sample $E \sim \ExpDist{ \sum_{i=1}^d \left(\ve^u  + \ve^l\right)_i}$ \\ 
    \uIf{$\tau_{\parent{j}} + E< \tau_j$}{
        Sample $\delta$ with probability proportional to $\ve^u  + \ve^l$ \\
        Sample a cut $\chi \sim \textsf{Uniform}\left(a,b\right)$ with 
        $a = u_{j\delta}^{\mX}, b = \vz_{\delta}$ if $\vz_{\delta} > u_{j\delta}$,
        else $a = \vz_{\delta}, b = l_{j\delta}^{\mX}$ \\
        Insert $j'$ ($j$ but below $\parent{j}$) to $j$ with: $N(j') = N(j) \cup \{\vz\}$,
        $\cdim_{j'} = \delta, \xi_{j'} = \chi, \tau_{j'} = \tau_{\parent{j}} + E,
        l_{j'}^{\mX} =  \min (l_j^{\mX}, \vz),
        u_{j'}^{\mX} = \max (u_j^{\mX}, \vz)$ \\
        Insert sibling $\sib{j}$ containing $\vz$ such that 
        $\lChild{j'} = j, \rChild{j'} = \sib{j}$ if $\vz_{\delta_j} > \xi_j$ or 
         $\rChild{j'} = j, \lChild{j'} = \sib{j}$, otherwise \\ 
        Set the Beta parameters according to number of points either side of $\xi_j$
    }
    \uElse{
        Update $l_{j}^{\mX} \leftarrow  \min (l_j^{\mX}, \vz),
        u_{j'}^{\mX} \leftarrow \max (u_j^{\mX}, \vz)$ \\ 
        \uIf{$j \in \leaves{\tree}$}{
            \Return{}
        }
        \uElse{
            \uIf{$\vz_{\cdim_j} < \xi_j$}{
                $\Child{j} = \lChild{j}$
            }
            \uElse{
                $\Child{j} = \rChild{j}$
            }
            Increment $\chi_{\Child{j}}$ by 1 \\ 
            \FSub{$T, \mX, \lambda, \vz, \Child{j}$}
        }
    }
    }
 \caption{ \MPT{} Insertion: $\MPTreeInsert{T}{\vz}$ }
 \label{alg: mpt-insertion}
\end{algorithm}
 
\begin{algorithm}[!htb]
 \KwIn{Mondrian Tree $T = (\tree, \CutDims, \CutLocs, \CutTimes)$}
 \KwOut{Mondrian Tree $T$ sampled over $\mX \setminus \vz$}
\SetKwFunction{FMain}{$\textsf{MT}_-$}
\SetKwFunction{FSub}{MTd}
\SetKwProg{Fn}{Function}{:}{}
\Fn{\FMain{$T, \mX, \lambda, \vz$}}{
$\Troot = \textsf{root}(T), \TreePath = \{ \Troot \} $\\
 \FSub{$T, \mX, \lambda, \vz, \Troot$} \\
 }
 \SetKwProg{Pn}{Function}{:}{\KwRet}
 \Pn{\FSub{$T, \mX, \lambda, \vz, j, \TreePath$}}{
 Find the deepest node $j$ containing $\vz$ \\ 
 Let $\TreePath = \{\Troot, u_1, u_2, \dots, u_k\}$ be the set of nodes from $\Troot$ to $j$ \\ 
 \For{$j \in \TreePath$}{
    Check if $\vz$ is \emph{internal} to the bounding box, or a point which defines
    the bounding box in one of dimensions $i \in \{1,2,\dots,d\}$ \\ 
    \uIf{$\vz$ is \emph{internal}}{
        $\Child{j} = \lChild{j}$ iff $\vz_{\delta_j} \le \xi_{j}$
        Decrement the cut and restriction counter corresponding to $\vz$ at $v$ by 1 \\
        Decrement all counts by 1 at the subtree rooted at $\Child{j}$ 
        \Comment{This only decrements on the side of the cut that $\vz$ should go.}\\
        \Return
    }
    \uElse{
        Let $B_j'$ be the bounding box at $j$ with $\vz$ ignored which has linear 
        dimension $L' = \LinDim{B_j'}$ \\
        Evaluate new node time $\tau_j'$ through inverse CDF \\
        \uIf{$\tau_j' \ge \lambda$}{
            Contract the entire subtree rooted at $j$ into $j$\\
            Set $\tau_j' = \lambda$ \\ 
            \Return
        }
    }
 }
 \Return{$T$}
    }
 \caption{Inplace deletions for the \MPT{}, $\MPTreeDelete{T}{\vz}$.}
 \label{alg: mpt-deletion}
\end{algorithm}

\newpage
\newpage

\section{Calculations for \Cref{fig:mpt-example-tree}} \label{sec:mpt-example-calcs}

Let us consider the generative process for sampling a 
\emph{streaming \MPT} (\acrshort{mpt}) to clarify the interplay between the underlying
Mondrian and \PTree{s}.
Let $\mX = \{\vx_1,\vx_2,\vx_3,\vx_4\}$ with $\vx_1 = (0,0), \vx_2=(1/4,1/4),
\vx_3=(2/5,4/5),\vx_4=(1,1)$.
Set $\gD = \bbox{\mX}$ to be the bounding box of the region containing $\mX$ and denote 
the two directions which span $\R^2$ be $x$ and $y$.
We show how sampling a depth 2 \emph{Mondrian Tree}
encodes a depth 4 \acrshort{mpt} which can be used to estimate the 
density over $\gD$.
Note that the only tree we store is the Mondrian Tree (with
lifetime $\lambda = \infty$), albeit with the extra parameters necessary
for \PTree{} density estimation.
Recall that the root of the tree is the node $\Troot$ which has an
empty index bitstring $b(\Troot) = \emptyset$ so it can be ignored
from node/parameter index strings.
The following example is illustrated in \Cref{fig:mpt-example-tree}.

Suppose the prior strength hyperparameter is $\gamma=1$
so it can be ignored from the \PTree{} calculations.
The first cut, $\xi_{\Troot}$ occurs at $x = 0.5$ and traverses the
entire bounding box $\gD$ in direction $x$.
This splits $\gD$ into two regions 
$R_{ 0 } \supset \{\vx_1,\vx_2,\vx_3\}$
and 
$R_{ 1} \supset \vx_4$,
each with volume $V_0,V_1 = 0.5$.
Hence, the posterior cut parameters are $\chi_0^* = 7/2$ \& $\chi_1^* = 3/2$
which results in $\mu_{\chi_{\Troot}} = 7/10$.

Next, call $\Restrict{\Troot}$  which
computes $\Restrict{R_{0}}$
\& $\Restrict{R_{ 1}}$ (see \Cref{alg: mpt}).
Since $\vx_4$ is isolated in $R_{ 1}$, the bounding box containing $\vx_4$ is supported on only 1 dimension;
$\Restrict{R_{ 1}} = (R_{1}, \emptyset)$ and the node 
storing $\vx_4$ is a \emph{Type I observation leaf} with volume $1/2$.
Hence, the \emph{mass} associated to this leaf is 
$P_{ 1} =  3/10$ and dividing
out the volume of the leaf yields the \emph{density} as $6/10$.

Let $B_{ 0} = \bbox{R_{ 0}}$ be the bounding box
containing the points $\vx_1,\vx_2,\vx_3$ from the region $R_{ 0}$.
Then, $\Restrict{R_{ 0}} = (B_{ 0}, B_{ 0}^C)$ which have volumes: 
$V_{B_{0}}= 32/100$ and $ V_{B_{ 0}^C} = 18/100$. 
The \text{\Polya} depth of this node is now 1 
(use restriction parameters \eqref{eq: mpt-prior} with $d_{\Troot} = 0$)
so the
(posterior) parameters for the split at this level are, for inclusion in $B_{0}$
(encoded with a $\in$) and exclusion (encoded with a $\neg$), respectively:
\begin{equation}
    \rho_{0\in}^* = 2^2 \cdot \nicefrac{32/100}{50/100} + 3 
    \qquad
    \text{and}
    \qquad
     \rho_{0\neg}^* = 2^2 \cdot \nicefrac{18/100}{50/100}.
\end{equation}
Accordingly, we obtain $\mu_{\rho_{0}} = 36/175$, `generate' an
\emph{internal node} with bitstring $0\in$ and a 
\emph{complementary leaf} with bitstring $0 \neg$.
Note that neither of these nodes is ever materialised as they 
are wholly defined by the node with index $b(j) = 0$ in the 
Mondrian Tree.
The masses allocated are $\mu_{\chi_0} \mu_{\rho_{0}}$ 
for the node $0 \in$ \&  $\mu_{\chi_0} (1 - \mu_{\rho_{0}})$
for the node $0 \neg$.

Since we have fixed a maximum depth of 2 for the Mondrian 
Tree, we perform one subsequent cut, $\xi_{ 0}$, to 
separate $\{\vx_1,\vx_2\}$ from $\{\vx_3\}$, and perform a final 
restriction procedure.
Hence, we have used the Mondrian Tree to correctly define
a \PTree{} over the partition of $\gD$.
Further details and calculations
can be found in \Cref{sec:mpt-example-calcs}.

Next, we deal with the points $\vx_1,\vx_2,\vx_3$ and the region $R_0$
generated left of the cut $\xi_{\Troot}$ by performing the first $\Restrict{\Troot}$ step.
Note that $\Restrict{\Troot}$ separately computes the restriction to 
bounding boxes either side of $\xi_{\Troot}$ noting that 
$\Restrict{R_0} = (B_0, B_0^C)$ and $\Restrict{R_1} = (R_1, \emptyset$).
Recall that $R_1$ contains a bounding box supported only on one dimension
so we set this to be a Type I observation leaf so the restriction simply
returns the set $R_1$.
On the other hand, consider $R_{\vx_1,\vx_2,\vx_3}$ which has a
volume of 1/2 and is decomposed into the subregions $B_0$ and
$R_{0} \setminus  B_0 = B_0^C$ (here the $B_0^C$
notation denotes the set complement in the universe $R_0$).
We thus obtain $V_{0\in} = \vol{B_0} = 32/100$ and 
$V_{0\neg} = 18/100$.
The \text{\Polya} depth of this node is now 1 so the
(posterior) parameters for the split at this level are, for inclusion in $B_0$
(encoded with a $\in$) and exclusion (encoded with a $\neg$), respectively:
\begin{align*}
    \rho_{0\in}^* &= (1+1)^2 \cdot \frac{32/100}{50/100} + 3 \\
    \rho_{0\neg}^* &= (1+1)^2 \cdot \frac{18/100}{50/100}.
\end{align*}
Overall, this results in $\mu_{\rho_{0}} = 36/175$, the internal node 
whose bitstring is $0\in$ and a \emph{complementary leaf} encoded by 
$0 \neg$.
The mass assigned to each of these nodes is 
$\mu_{\chi_0} \mu_{\rho_{0}}$ and  
$\mu_{\chi_0} (1 - \mu_{\rho_{0}})$, respectively.

Following this restriction, we complete one more cut: $\xi_{0\in}$
at $y=0.4$ defined only on the box $B_0$ and generate
the two regions $R_{00}, R_{01}$.
Since $|R_{01} \cap \mX| = 1$, we terminate the process and treat this leaf as an observed leaf of type I with mass 
$\mu_{\chi_0} \mu_{\rho_{0}} (1 - \mu_{\chi_{0\in}})$.
On the other hand, $|R_{00} \cap \mX| = 2$ so we again perform
$\Restrict{R_{0\in0}} = (B_{00}, B_{00}^C)$
which returns the bounding box $B_{00} = \bbox{ {\vx_1, \vx_2} }$
in an observation leaf of type II, along with its complementary region
which is added to the set of complementary leaves.
At this point we terminate the process, so there are 5 leaves generated which partition the entire input domain as defined by the input data.

\paragraph{Numerics.}
Given data $\mX$ and the cuts $\xi_{\Troot}, \xi_{0}$ the following 
quantities are used to evaluate the density in each of the 5 leaves:
\begin{itemize}
    \item{Cut at root node \Polya{} depth = 0: 
    $\chi_0^* = \frac{1/2}{1/2}+3$ and 
    $\chi_0^* = \frac{1/2}{1/2}+1$ so that 
    $\mu_{\chi_{\Troot}} = 7/10$ and nodes $0,1$ are created.
    They are internal and observation leaf Type I, respectively.}
    
    \item{Restrict at node $0$, \Polya{} depth = 1:
    $ \rho_{0\in}^* = (1+1)^2 \cdot \frac{32/100}{50/100} + 3$,
    $ \rho_{0\neg}^* = (1+1)^2 \cdot \frac{18/100}{50/100}$ so that 
    $ \mu_{\rho_{0}} = 36/175$.
    Internal node $0\in$ and complementary leaf $0\neg$ are created.}
    
    \item{Cut at node $0\in$, \Polya{} depth = 2:
    $\chi_{0\in0}^* = 9 \frac{16}{32} + 2 = 13/2$,
    $\chi_{0\in1}^* = 9 \frac{16}{32} + 1 = 11/2$, so that 
    $\mu_{\chi_{0 \in}} = 13 / 24$.
    Internal node $0\in 0 $ and $0\in 1$ are created, however, $0\in 1$
    has exactly one datapoint in so is a Type I observation leaf.
    }
    
    \item{Restrict at node $0\in$ with \Polya{} depth=3:
    $\rho_{0\in0\in} = 4^2 \frac{1/16}{16/100} + 2$,
    $\rho_{0\in0\neg} = 4^2 \frac{16/100 - 1/16}{16/100}$ so that 
    $\mu_{\rho_{0 \in}} = 11/24$ to get the Type II observation leaf
    containing $B(\vx_1, \vx_2)$ and the complementary leaf.}
\end{itemize}

\section{Further Details: \Cref{sec: anomaly-detection}} 
\label{sec: anom-appendix}

The dataset details are given in \Cref{table: data-details}. 
We then present the numeric results corresponding to \Cref{tab:mean_ranks} in 
\Cref{tab:pidforest-baseline} and a discussion in the subsequent section.
We also briefly present some results on the running time as well as an
initial statistical analysis.

\begin{table}[htb]
\centering
\begin{tabular}{lllll}
\hline
Dataset            & $n$    & $d$          & Number of Anomalies & \% Anomalies \\ \hline
\multicolumn{5}{c}{PidForest Baseline Comparision: PyOD}                        \\ \hline
Thyroid            & 3772   & 6            & 93                  & 2.5          \\
Mammography        & 11183  & 6            & 260                 & 2.32         \\
Seismic            & 2584   & 11           & 170                 & 6.5          \\
Satimage-2         & 5803   & 36           & 71                  & 1.2          \\
Vowels             & 1456   & 12           & 50                  & 3.4          \\
Musk               & 3062   & 166          & 97                  & 3.2          \\
HTTP (KDDCUP99)    & 567479 & 3            & 2211                & 0.4          \\
SMTP (KDDCUP99)    & 95156  & 3            & 30                  & 0.03         \\ \hline
\multicolumn{5}{c}{PidForest Baseline Comparision: NAB}                         \\ \hline
A.T                & 7258   & 10 (Shingle) & 726                 & 10.0         \\
CPU                & 18041  & 10 (Shingle) & 1499                & 8.3          \\
M.T                & 22686  & 10 (Shingle) & 2268                & 10.0         \\
NYC                & 10311  & 10 (Shingle) & 1035                & 10.0         \\ \hline
\multicolumn{5}{c}{All other PyOD Datasets}                                     \\ \hline
Annthyroid         & 7200   & 6            & 534                 & 7.42         \\
Arrhythmia         & 452    & 274          & 66                  & 15           \\
BreastW            & 683    & 9            & 239                 & 35           \\
Cardio             & 1831   & 21           & 176                 & 9.6          \\
Ecoli              & 336    & 7            & 9                   & 2.6          \\
ForestCover        & 286048 & 10           & 2747                & 0.9          \\
Glass              & 214    & 9            & 9                   & 4.2          \\
Heart              & 349    & 44           & 95                  & 27.7         \\
Ionosphere         & 351    & 33           & 126                 & 36           \\
Letter Recognition & 1600   & 32           & 100                 & 6.25         \\
Lympho             & 148    & 18           & 6                   & 4.1          \\
Mnist              & 7603   & 100          & 700                 & 9.2          \\
Mulcross           & 262144 & 4            & 26214               & 10           \\
Optdigits          & 5216   & 64           & 150                 & 3            \\
Pendigits          & 6870   & 16           & 156                 & 2.27         \\
Pima               & 768    & 8            & 268                 & 35           \\
Satellite          & 6435   & 36           & 2036                & 32           \\
Shuttle            & 49097  & 9            & 3511                & 7            \\
Speech             & 3686   & 400          & 61                  & 1.65         \\
Vertebral          & 240    & 6            & 30                  & 12.5         \\
WBC                & 278    & 30           & 21                  & 5.6          \\
Wine               & 129    & 13           & 10                  & 7.7          \\
Yeast              & 1364   & 8            & 64                  & 4.7          \\ \hline
\multicolumn{5}{c}{Other NAB Datasets}                                          \\ \hline
ad\_exchange       & 1634   & 10 (Shingle)           & 166                 & 10           \\
aws\_cloud\_cpu    & 4023   & 10 (Shingle)           & 402                 & 10           \\
google\_tweets     & 15833  & 10 (Shingle)           & 1432                & 10           \\
rogue\_hold        & 1873   & 10 (Shingle)           & 190                 & 10           \\
rogue\_updown      & 5306   & 10 (Shingle)           & 530                 & 10           \\
speed              & 2486   & 10 (Shingle)           & 250                 & 10           \\ \hline
\end{tabular}
\caption{Data Summary. For the Heart dataset we used the OpenML version \cite{OpenML2013}.
The bottom panel are streaming datasets from \cite{NAB}.}
\label{table: data-details}
\end{table}

\subsection{Experimental Results}
The experimental setup is as in Section \ref{sec: anomaly-detection} and the 
\acrshort{auc} is recorded for each dataset.
We separately test the batch (\iforest{}, \pid{} and \acrshort{bmpf}) and 
streaming methods (\acrshort{smpf}, \rrcf{}).
The results are given in \Cref{tab:pidforest-baseline}: 5 independent trials are 
performed for each dataset with the mean and standard deviation being reported.
We boldface the winner for every dataset and this is used to evaluate the 
mean rank and number of wins from \Cref{tab:mean_ranks}. 
Note that \Cref{tab:mean_ranks} is evaluated for every trial over all datasets, 
whereas \Cref{tab:pidforest-baseline} simply records the winner for the best 
reported mean \acrshort{auc}.
The general behaviour is that both of the Mondrian \Polya{} Forests behave comparably
prior state-of-the-art methods.

\begin{table}[htbp]
\centering
\begin{tabular}{llllll}
\hline
 &
  \acrshort{bmpf} &
  iForest &
  PiDForest &
  \acrshort{smpf} &
  RRCF \\ \hline
\multicolumn{6}{c}{PidForest Baseline Comparision: PyOD} \\ \hline
Thyroid &
  \textbf{0.950  $\pm$ 0.007} &
  0.805  $\pm$ 0.033 &
  0.843  $\pm$ 0.014 &
  \textbf{0.948  $\pm$ 0.004} &
  0.744  $\pm$ 0.006 \\
Mammography &
  \textbf{0.869  $\pm$ 0.007} &
  0.860  $\pm$ 0.004 &
  0.858  $\pm$ 0.011 &
  \textbf{0.866  $\pm$ 0.004} &
  0.831  $\pm$ 0.003 \\
Seismic &
  0.697  $\pm$ 0.007 &
  \textbf{0.714  $\pm$ 0.009} &
  0.710  $\pm$ 0.011 &
  0.621  $\pm$ 0.015 &
  \textbf{0.699  $\pm$ 0.006} \\
Satimage-2 &
  0.991  $\pm$ 0.001 &
  \textbf{0.992  $\pm$ 0.005} &
  0.992  $\pm$ 0.004 &
  0.986  $\pm$ 0.001 &
  \textbf{0.991  $\pm$ 0.003} \\
Vowels &
  \textbf{0.777  $\pm$ 0.025} &
  0.772  $\pm$ 0.024 &
  0.748  $\pm$ 0.003 &
  0.757  $\pm$ 0.020 &
  \textbf{0.817  $\pm$ 0.005} \\
Musk &
  \textbf{1.000  $\pm$ 0.001} &
  \textbf{1.000  $\pm$ 0.000} &
  \textbf{1.000  $\pm$ 0.000} &
  0.972  $\pm$ 0.014 &
  \textbf{0.998  $\pm$ 0.001} \\
HTTP &
  0.996  $\pm$ 0.000 &
  0.997  $\pm$ 0.004 &
  \textbf{0.998  $\pm$ 0.003} &
  \textbf{0.997  $\pm$ 0.000} &
  0.993  $\pm$ 0.000 \\
SMTP &
  0.835  $\pm$ 0.014 &
  \textbf{0.919  $\pm$ 0.003} &
  0.919  $\pm$ 0.006 &
  0.836  $\pm$ 0.009 &
  \textbf{0.886  $\pm$ 0.017} \\ \hline
\multicolumn{6}{c}{PidForest Baseline Comparision: NAB} \\ \hline
NYC &
  0.527  $\pm$ 0.000 &
  \textbf{0.546  $\pm$ 0.082} &
  0.545  $\pm$ 0.082 &
  \textbf{0.558  $\pm$ 0.000} &
  0.537  $\pm$ 0.004 \\
A.T &
  \textbf{0.785  $\pm$ 0.006} &
  0.731  $\pm$ 0.098 &
  0.730  $\pm$ 0.096 &
  \textbf{0.773  $\pm$ 0.016} &
  0.693  $\pm$ 0.008 \\
CPU &
  \textbf{0.913  $\pm$ 0.002} &
  0.818  $\pm$ 0.149 &
  0.815  $\pm$ 0.148 &
  \textbf{0.911  $\pm$ 0.002} &
  0.786  $\pm$ 0.004 \\
M.T &
  \textbf{0.822  $\pm$ 0.003} &
  0.740  $\pm$ 0.138 &
  0.740  $\pm$ 0.138 &
  \textbf{0.820  $\pm$ 0.007} &
  0.749  $\pm$ 0.005 \\ \hline
\multicolumn{6}{c}{All other PyOD Datasets} \\ \hline
Annthyroid &
  0.663 $\pm$ 0.012 &
  0.809 $\pm$ 0.014 &
  \textbf{0.880 $\pm$ 0.008} &
  0.663 $\pm$ 0.013 &
  \textbf{0.741 $\pm$ 0.004} \\
Arrhythmia &
  \textbf{0.813 $\pm$ 0.010} &
  0.799 $\pm$ 0.009 &
  - &
  0.549 $\pm$ 0.033 &
  \textbf{0.787 $\pm$ 0.002} \\
Breastw &
  0.973 $\pm$ 0.001 &
  \textbf{0.986 $\pm$ 0.001} &
  0.973 $\pm$ 0.001 &
  \textbf{0.979 $\pm$ 0.004} &
  0.644 $\pm$ 0.004 \\
Cardio &
  0.910 $\pm$ 0.016 &
  \textbf{0.923 $\pm$ 0.005} &
  0.860 $\pm$ 0.012 &
  0.873 $\pm$ 0.038 &
  \textbf{0.898 $\pm$ 0.004} \\
Cover &
  0.772 $\pm$ 0.018 &
  \textbf{0.910 $\pm$ 0.000} &
  0.841 $\pm$ 0.000 &
  \textbf{0.741 $\pm$ 0.044} &
  0.674 $\pm$ 0.005 \\
Ecoli &
  \textbf{0.881 $\pm$ 0.012} &
  0.857 $\pm$ 0.006 &
  0.859 $\pm$ 0.007 &
  \textbf{0.900 $\pm$ 0.039} &
  0.858 $\pm$ 0.002 \\
Glass &
  \textbf{0.798 $\pm$ 0.006} &
  0.708 $\pm$ 0.008 &
  0.690 $\pm$ 0.023 &
  \textbf{0.824 $\pm$ 0.018} &
  0.721 $\pm$ 0.013 \\
Heart &
  0.203 $\pm$ 0.012 &
  \textbf{0.251 $\pm$ 0.010} &
  0.233 $\pm$ 0.033 &
  \textbf{0.237 $\pm$ 0.019} &
  0.210 $\pm$ 0.010 \\
Ionosphere &
  \textbf{0.877 $\pm$ 0.004} &
  0.855 $\pm$ 0.005 &
  0.844 $\pm$ 0.014 &
  0.891 $\pm$ 0.005 &
  \textbf{0.896 $\pm$ 0.002} \\
Letter &
  0.621 $\pm$ 0.022 &
  0.633 $\pm$ 0.015 &
  \textbf{0.643 $\pm$ 0.025} &
  0.623 $\pm$ 0.021 &
  \textbf{0.735 $\pm$ 0.008} \\
Lympho &
  0.984 $\pm$ 0.007 &
  \textbf{0.997 $\pm$ 0.003} &
  \textbf{0.997 $\pm$ 0.003} &
  0.975 $\pm$ 0.017 &
  \textbf{0.993 $\pm$ 0.001} \\
Mnist &
  \textbf{0.807 $\pm$ 0.022} &
  0.804 $\pm$ 0.009 &
  - &
  \textbf{0.812 $\pm$ 0.037} &
  0.770 $\pm$ 0.004 \\
Optdigits &
  0.704 $\pm$ 0.044 &
  \textbf{0.706 $\pm$ 0.026} &
  - &
  \textbf{0.650 $\pm$ 0.150} &
  0.529 $\pm$ 0.013 \\
Pendigits &
  0.929 $\pm$ 0.006 &
  \textbf{0.952 $\pm$ 0.006} &
  0.947 $\pm$ 0.010 &
  \textbf{0.913 $\pm$ 0.004} &
  0.869 $\pm$ 0.011 \\
Pima &
  0.658 $\pm$ 0.006 &
  \textbf{0.680 $\pm$ 0.015} &
  0.679 $\pm$ 0.012 &
  \textbf{0.601 $\pm$ 0.005} &
  0.593 $\pm$ 0.005 \\
Satellite &
  0.704 $\pm$ 0.005 &
  \textbf{0.717 $\pm$ 0.021} &
  0.697 $\pm$ 0.031 &
  \textbf{0.719 $\pm$ 0.012} &
  0.684 $\pm$ 0.003 \\
Shuttle &
  0.505 $\pm$ 0.000 &
  \textbf{0.997 $\pm$ 0.000} &
  0.988 $\pm$ 0.011 &
  0.506 $\pm$ 0.000 &
  \textbf{0.909 $\pm$ 0.004} \\
Speech &
  0.475 $\pm$ 0.017 &
  0.474 $\pm$ 0.018 &
  \textbf{0.484 $\pm$ 0.015} &
  \textbf{0.492 $\pm$ 0.033} &
  0.470 $\pm$ 0.023 \\
Vertebral &
  0.352 $\pm$ 0.034 &
  \textbf{0.359 $\pm$ 0.006} &
  0.332 $\pm$ 0.034 &
  \textbf{0.394 $\pm$ 0.032} &
  0.390 $\pm$ 0.004 \\
WBC &
  \textbf{0.950 $\pm$ 0.005} &
  0.941 $\pm$ 0.007 &
  0.945 $\pm$ 0.010 &
  \textbf{0.925 $\pm$ 0.008} &
  0.921 $\pm$ 0.004 \\
Wine &
  \textbf{0.951 $\pm$ 0.005} &
  0.746 $\pm$ 0.025 &
  0.777 $\pm$ 0.037 &
  0.882 $\pm$ 0.013 &
  \textbf{0.962 $\pm$ 0.003} \\
Yeast &
  0.989 $\pm$ 0.000 &
  \textbf{0.996 $\pm$ 0.001} &
  0.990 $\pm$0.008 &
  \textbf{0.990 $\pm$ 0.001} &
  0.980 $\pm$ 0.002 \\
  \hline
\multicolumn{6}{c}{Other NAB Datasets} \\
\hline
ad\_exchange &
  0.621 $\pm$ 0.005 &
  \textbf{0.665 $\pm$ 0.004} &
  0.660 $\pm$ 0.006 &
  0.625 $\pm$ 0.010 &
  \textbf{0.635 $\pm$ 0.006} \\
aws\_cloud\_cpu &
  \textbf{0.608 $\pm$ 0.005} &
  0.561 $\pm$ 0.008 &
  0.574 $\pm$ 0.006 &
  0.587 $\pm$ 0.002 &
  \textbf{0.601 $\pm$ 0.003} \\
google\_tweets &
  \textbf{0.645 $\pm$ 0.007} &
  0.573 $\pm$ 0.007 &
  0.620 $\pm$ 0.007 &
  0.632 $\pm$ 0.017 &
  \textbf{0.637 $\pm$ 0.008} \\
rogue\_hold &
  0.399 $\pm$ 0.004 &
  \textbf{0.474 $\pm$ 0.009} &
  0.452 $\pm$ 0.004 &
  0.399 $\pm$ 0.009 &
  \textbf{0.480 $\pm$ 0.001} \\
rogue\_updown &
  0.497 $\pm$ 0.000 &
  0.494 $\pm$ 0.008 &
  \textbf{0.499 $\pm$ 0.000} &
  \textbf{0.501 $\pm$ 0.005} &
  0.493 $\pm$ 0.000 \\
speed &
  0.557 $\pm$ 0.014 &
  0.557 $\pm$ 0.006 &
  \textbf{0.566 $\pm$ 0.005} &
  \textbf{0.568 $\pm$ 0.018} &
  0.548 $\pm$ 0.010 \\ \hline
Num. AUC wins &
  16 &
  19 &
  8 &
  23 &
  17 \\ \hline
\end{tabular}
\caption{Anomaly detection experiments from \Cref{sec: anomaly-detection}. 
The top two panels ``PidForest Baseline Comparison\dots'' 
is a direct comparison to 
Table 1 of \cite{gopalan2019pidforest}.
The middle panel is all other \acrshort{pyod} datasets and the bottom
panel is a selection of other datasets from the NAB repository.
Columns with ``-''for 
\pid{} indicate failed executions due to the error ``No entropy in chosen feature''.
The three leftmost methods are the batch forests, while the two right most methods
are the streaming methods.
Winners are written in boldface, the batch methods are compared against one 
another separately from the streaming methods.}
\label{tab:pidforest-baseline}
\end{table}

\subsection{Classical Batch Methods}
Anomaly detection is a classification problem with
imbalanced classes consisting of a (large) `normal' subset of data, and
a small subset containing anomalies.
One could adapt supervised learning techniques (e.g a One-Class 
Support Vector Machines (\acrshort{1csvm}) 
\cite{scholkopf2001estimating}) but
labelling anomalies is time-consuming \& expensive so
supervised learning is incompatible with the large-scale streaming model.
For instance, 
training a \acrshort{1csvm} takes time between
$O(n^2)$ and $O(n^3)$ depending on
the sizes of $n$ and $D$ \cite{bottou2007support}.
Unsupervised methods have also been proposed which rely on some
notion of local or global clustering.
For example, Local Outlier Factor
(\acrshort{lof}) \cite{breunig2000lof};
$k$-Nearest Neighbours (\acrshort{knn}) (\cite{ramaswamy2000efficient}, \cite{angiulli2002fast});
or Principal Components Analysis (\acrshort{pca}),
(\cite{shyu2003novel}, \cite{aggarwal2015outlier}).
However, the time complexity of these methods can scale
quadratically with $n$ or $D$ so are unsuitable in the 
large-scale or high-dimensional setting.

We are interested in \emph{unsupervised} methods: typically, these 
approaches rely on some notion of local or global clustering, for example Local Outlier Factor 
(\acrshort{lof}) \cite{breunig2000lof}, $k$-Nearest Neighbours (\acrshort{knn}) (\cite{ramaswamy2000efficient}, \cite{angiulli2002fast}),
or Principal Components Analysis (\acrshort{pca}),
(\cite{shyu2003novel}, \cite{aggarwal2015outlier}).
These solutions do not scale for large-scale and high-dimensional datasets in the offline 
setting, let alone when we are constrained to the data stream model;
consider input data $\mX \in \R^{n \times D}$,
\acrshort{lof} requires time at least $\Omega(n)$, 
but for high dimensions requires $\Theta(n^2)$ 
time \cite{breunig2000lof}.
Additionally, \acrshort{pca} requires 
a singular value decomposition (\acrshort{svd}) which takes
time $O(nD^2)$. 
Using these datasets in the large-scale batch setting is problematic because of 
the overhead incurred, let alone when we are further constrained to the 
streaming environment.
Due to the scalability of the batch offline methods, we only present the 
results on a small subset of the datasets tested: these are given in 
\Cref{tab:non-rf-baseline}.

\begin{table}[htbp]
    \centering
\begin{tabular}{lllll}
\toprule
                                      &   SVM &   LOF & kNN             & PCA             \\
\midrule
 http                                 & 0.231 & 0.996 & $0.353$ & $\bm{0.999}$ \\
 mammography                          & 0.839 & 0.886 & $0.720$ & $\bm{0.872}$ \\
 musk                                 & 0.373 & \textbf{1.000} & $0.416$ & $\bm{1.000}$ \\
 satimage-2                           & 0.936 & \textbf{0.977} & $0.540$ & $0.996$ \\
 siesmic                              & \textbf{0.740} & 0.682 & $0.553$ & $0.589$ \\
 smtp                                 & 0.895 & 0.823 & $\bm{0.904}$ & $0.898$ \\
 thyroid                              & \textbf{0.751} & 0.673 & $0.737$ & $0.573$ \\
 vowels                               & \textbf{0.975} & 0.606 & $0.943$ & $0.778$ \\
 nyc\_taxi                             & \textbf{0.697} & 0.511 & $0.671$ & $0.453$ \\
 ambient\_temperature\_system\_failure   & 0.634 & \textbf{0.792} & $0.563$ & $0.783$ \\
 cpu\_utilization\_asg\_misconfiguration & 0.724 & 0.858 & $0.560$ & $\bm{0.898}$ \\
 machine\_temperature\_system\_failure   & 0.759 & \textbf{0.834} & $0.501$ & $0.822$ \\
\bottomrule
\end{tabular}
    \caption{Baseline Experiments. Non Random Forest Methods}
    \label{tab:non-rf-baseline}
\end{table}

 \subsection{Running Time}
 \begin{table}[]
    \centering
    \begin{tabular}{llll}
\toprule
                                      & \acrshort{smpf}   & \pid           & Approx. Speedup   \\
\midrule
thyroid                              & $5.4 \pm 0.1$           & $25.8 \pm 27.2$   & $5.0$     \\
mammography                          & $5.0 \pm 0.1$           & $9.5 \pm 0.1$     & $2.0$     \\
seismic                              & $1.0 \pm 0.0$           & $26.4 \pm 0.5$    & $28.0$    \\
satimage-2                           & $22.0 \pm 0.6$          & $25.5 \pm 0.5$    & $1.0$     \\
vowels                               & $9.3 \pm 0.2$           & $22.4 \pm 0.4$    & $2.0$     \\
musk                                 & $53.8 \pm 3.0$          & $170.3 \pm 4.4$   & $3.0$     \\
http                                 & $91.5 \pm 3.3$          & $260.9 \pm 3.9$   & $3.0$     \\
smtp                                 & $20.4 \pm 0.9$          & $370.4 \pm 646.0$ & $18.0$    \\
\midrule 
NYC                             & $15.9 \pm 0.0$          & $24.8 \pm 0.6$    & $2.0$     \\
A.T   & $6.6 \pm 0.1$           & $22.5 \pm 1.0$    & $3.0$     \\
CPU & $9.2 \pm 0.1$           & $22.3 \pm 0.2$    & $2.0$     \\
M.T   & $11.3 \pm 0.6$          & $29.1 \pm 0.9$    & $3.0$     \\
\midrule 
 annthyroid                           & $7.5 \pm 0.4$           & $11.5 \pm 0.7$    & $2.0$     \\
 arrhythmia                           & $1.6 \pm 0.0$           & -                 & -     \\
 breastw                              & $3.0 \pm 0.1$           & $8.2 \pm 0.1$     & $3.0$     \\
 cardio                               & $1.0 \pm 0.0$           & $17.4 \pm 0.2$    & $17.0$    \\
 cover                                & $297.7 \pm 6.1$         & $1181.4 \pm 0.0$  & $4.0$     \\
 ecoli                                & $0.2 \pm 0.1$           & $10.1 \pm 0.3$    & $42.0$    \\
 glass                                & $1.6 \pm 0.6$           & $9.5 \pm 0.2$     & $6.0$     \\
 heart                                & $7.0 \pm 0.1$           & $22.8 \pm 0.0$    & $3.0$     \\
 ionosphere                           & $4.1 \pm 0.3$           & $35.9 \pm 0.4$    & $9.0$     \\
 letter                               & $9.4 \pm 0.3$           & $26.7 \pm 0.2$    & $3.0$     \\
 lympho                               & $0.3 \pm 0.0$           & $9.0 \pm 0.1$     & $29.0$    \\
 mnist                                & $3.6 \pm 0.2$           & -                 & -     \\
 optdigits                            & $2.4 \pm 0.1$           & -                 & -    \\
 pendigits                            & $17.0 \pm 0.3$          & $21.0 \pm 0.3$    & $1.0$     \\
 pima                                 & $4.1 \pm 0.1$           & $8.8 \pm 0.1$     & $2.0$     \\
 satellite                            & $25.3 \pm 0.7$          & $58.8 \pm 1.1$    & $2.0$     \\
 shuttle                              & $14.9 \pm 0.3$          & $209.4 \pm 2.4$   & $14.0$    \\
 speech                               & $188.7 \pm 3.1$         & $618.3 \pm 1.2$   & $3.0$     \\
 vertebral                            & $3.0 \pm 0.1$           & $11.2 \pm 0.2$    & $4.0$     \\
 wbc                                  & $7.8 \pm 0.7$           & $50.8 \pm 0.6$    & $7.0$     \\
 wine                                 & $3.3 \pm 0.0$           & $15.3 \pm 0.1$    & $5.0$     \\
 yeast                                & $0.5 \pm 0.0$           & $14.2 \pm 0.9$    & $28.0$    \\
\bottomrule
\end{tabular}
    \caption{Runtime comparison in wallclock time (seconds) for completion.
    Panes separated as in \Cref{tab:pidforest-baseline}.
    \acrshort{smpf} contains $n$ points per tree whereas
    \pid{} contains only 100 points per tree.}
    \label{table:runtime-comparison}
\end{table}

Although not the focus of this investigation, we present an interesting 
contrast between our method and \pid~ in terms of running time.
These results are summarised in \Cref{table:runtime-comparison} in which the wall 
clock time necessary to perform the forest sampling from the previous
experiment (\Cref{tab:pidforest-baseline}) is recorded.
We compare only \acrshort{smpf} and \pid~ as both \rrcf{} and \iforest{}
are heavily optimised and the other methods are not suitable for 
streaming data.
Recall that our algorithm uses all datapoints in $\mX$ to 
(i) cut the data at random,
(ii) update model parameters for probability mass estimation.
While the cutting is cheap, it is likely that the cuts may not be 
informative which is why the second corrective step is required.

\pid{} takes a complementary approach by optimising for the 
cut at every level rather than cutting at random, using only a small 
subset of the data to build the tree.
Our findings suggest that it is more efficient to make random cuts and 
update the parameters of the density model than solving the optimisation
problem for \pid{}.
This is borne out in \Cref{table:runtime-comparison}, \Cref{sec: anom-appendix} where our streaming 
implementation of \acrshort{mpf} is at least a (small) constant factor
quicker than
\pid{}, but can reach almost 50x (approximate) speedup over the time it 
takes to fit a \pid{}.
Of further interest is the fact that we use \emph{all} datapoints per 
tree, whereas \pid{} uses only 100 points per tree meaning that,
in aggregate, our method is substantially faster.
While both implementations of \acrshort{smpf} and \pid{} are 
proof-of-concept, the similarity of our proposed \acrshort{bmpf} and 
\acrshort{smpf} to the \iforest{} and \rrcf{} suggests that it should be
substantial room for improvement, achieving runtime comparable to the best
implementations of each.

\subsection{Statistical Analysis}

We use repeated measures ANOVA as an omnibus test to determine if there are any significant differences between the mean values of the populations, shown in \Cref{tab:anova}. We reject the null hypothesis ($F=104.844$, $p<0.001$) of the repeated measures ANOVA that there is a difference between the mean values of the for the independent variable of algorithm (the dataset and interaction were also significant). Therefore, we assume that there is a statistically significant difference between the mean values of the populations. Given that the results of the ANOVA test are significant, for post-hoc testing we use the paired two-way t-tests to infer which differences are significant. 
The results are shown in \Cref{tab:post_hoc}. The results at the $p<0.01$ level that show that bMPF and iForest both significantly outperform sMPF, all methods significantly outperform RRCF. All other comparisons failed to reach significance, indicating that based on these experiments, these methods cannot be separated from one another.

\begin{table}[htbp]
\label{tab:anova}
\caption{2-way repeated measures ANOVA (F-statistic) for the main effects of algorithm, dataset, and interaction effects. ddof1/ddof2 are the degrees of freedom for the factor/replicates.}
\centering
\begin{tabular}{lrrrrrrrrr}
\toprule
 Source              &   ddof1 &   ddof2 &        F &   p \\
\midrule
 algorithm           &       4 &      16 &  104.844 &   0.000 \\
 dataset             &      11 &      44 & 4216.040 &   0.000 \\
 algorithm * dataset &      44 &     176 &  104.563 &   0.000 \\
\bottomrule
\end{tabular}
\end{table}

\begin{table}[htbp]
\label{tab:post_hoc}
\caption{Post-hoc paired 2-sample t-tests for the main effect of algorithm. Bold results indicate significance at the $p<0.01$ level.}
\centering
\begin{tabular}{llrrrr}
\toprule
 A       & B         &       T &   p &    BF10 &   hedges \\
\midrule
 sMPF    & bMPF      &  -5.015 &   {\bf 0.007} &   8.755 &   -3.056 \\
 sMPF    & RRCF      &  13.466 &   {\bf 0.000} & 135.300 &    8.626 \\
 sMPF    & iForest   &  -4.761 &   {\bf 0.009} &   7.669 &   -2.181 \\
 sMPF    & PiDForest &  -0.426 &   0.692 &   0.428 &   -0.279 \\
 bMPF    & RRCF      &  14.553 &   {\bf 0.000} & 169.618 &   10.640 \\
 bMPF    & iForest   &   3.262 &   0.031 &   3.120 &    1.710 \\
 bMPF    & PiDForest &   3.635 &   0.022 &   3.985 &    2.744 \\
 RRCF    & iForest   & -22.731 &   {\bf 0.000} & 629.989 &  -13.034 \\
 RRCF    & PiDForest & -18.685 &   {\bf 0.000} & 353.026 &   -8.654 \\
 iForest & PiDForest &   2.477 &   0.068 &   1.777 &    1.758 \\
\bottomrule
\end{tabular}

\end{table}

\subsection{NAB Datasets}
The result in \Cref{tab:pidforest-baseline} often suggest that the \acrshort{auc} for the NAB datasets can be relatively low.
Additionally, sometimes our streaming method appears to lose out to the 
\rrcf{} approach.
We suggest that part of the reason here for the slightly diminished 
\acrshort{auc} performance could be to do with the labelling of the 
NAB datasets.
The anomalies are not labelled as specific datapoints, but rather windows
or intervals which contain an anomaly.
This can clearly hurt the performance of a detector as not detecting an
anomaly at the start of a window (which may well be \emph{normal}
behaviour) would be recorded as incorrect predictions in the NAB labelling
scheme.
Likewise, the same applies if a detector quickly returns to normal 
behaviour after the anomaly despite the labelling suggesting that the data
index still lies in an anomalous window.
Both of these behaviours are observed in Figures \ref{fig:trace-rogue} and
\ref{fig:trace-adexch}.

\begin{figure}
    \centering
    \includegraphics[width=0.9\linewidth]{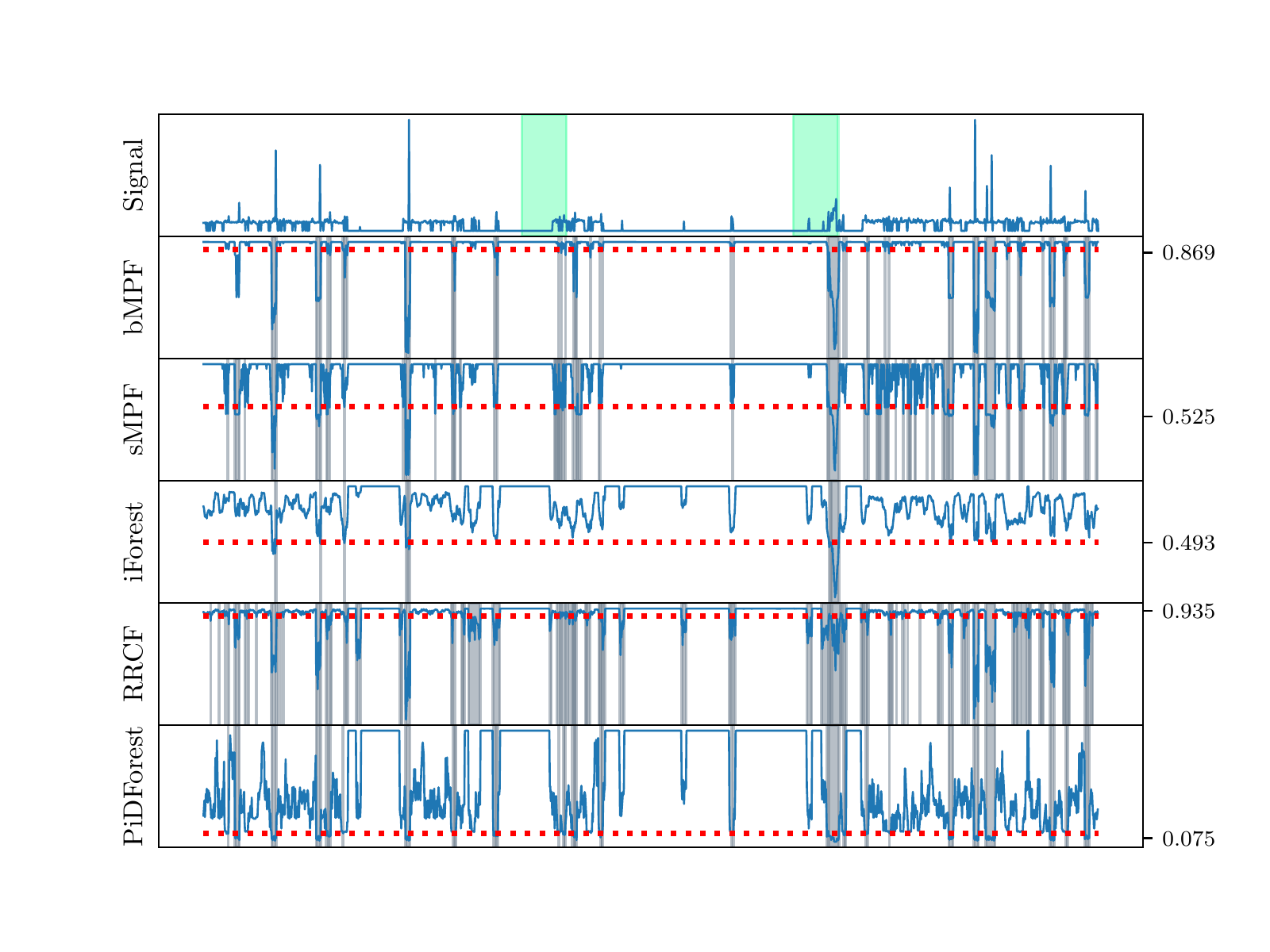}
    \caption{``Rogue\_hold'' trace denoted by the blue curve in each panel.
    The top panel illustrates the ground truth anomalies with their 
    associated window in green.
    Flagged anomalies are in the grey shading and the red dashed line
    is the threshold which achieves the optimum \acrshort{auc}.
    }
    \label{fig:trace-rogue}
\end{figure}

\begin{figure}
    \centering
    \includegraphics[width=0.9\linewidth]{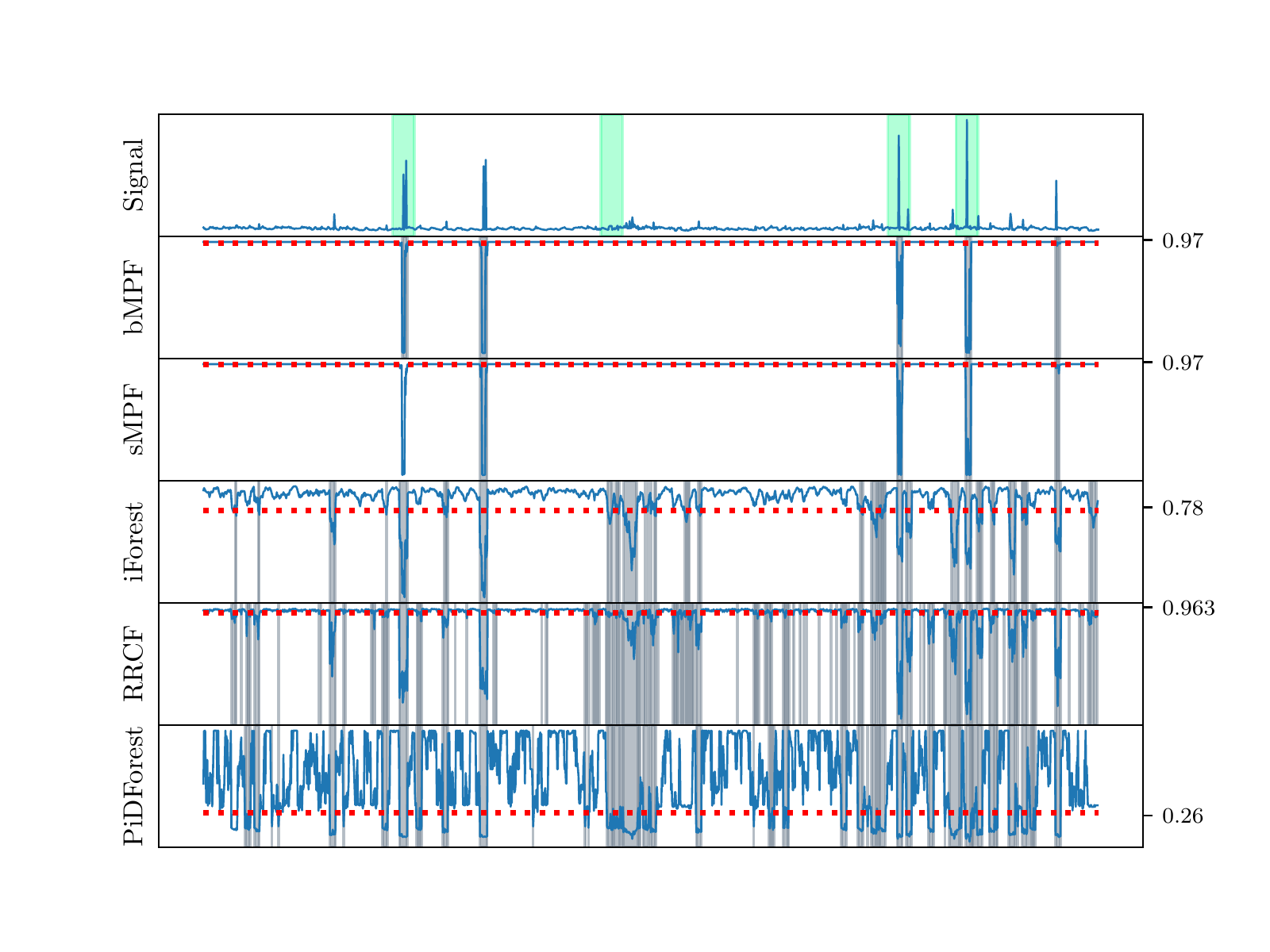}
    \caption{``Ad\_exchange'' trace.
    Plots as described in Figure \ref{fig:trace-rogue}}
    \label{fig:trace-adexch}
\end{figure}

\newpage
\section{Illustrative Examples}
\label{sec: synthetic-examples}

\subsection{2d Toy Datasets}
We provide a simple comparison of the methods on all of the baseline
synthetic examples taken from the scikit-learn outlier detection 
page \cite{outlier-sklearn} which contains unimodal and bimodal data.
The data is of size $n=500$ which is split between 
$n_{\text{inliers}} = 425$ inlier points and the remaining 
$n_{\text{outliers}} =75$ being planted outliers chosen uniformly over 
the input domain.
For visual comparison, we plot the resulting classification induced by 
each of the random forest methods at the optimum threshold.
The results are illustrated in Figure \ref{fig:sklearn-outlier}.
We additionally record the area under the ROC curve (AUC) and 
area under the
precision-recall-gain (PRG) curve in Table \ref{table:synthetic-auprg}
\cite{flach2015precision}.
Area under a precision-recall curve is not justified, instead use area under the PRG curve.
We use \cite{prg} to evaluate the Precision-Recall-Gain and observe that
again our methods perform well compared to other random forests.
These results are presented in \Cref{table:synthetic-auprg} but a more
in-depth study is deferred for future work.
\begin{figure}
    \centering
 \includegraphics[width=0.9\linewidth]{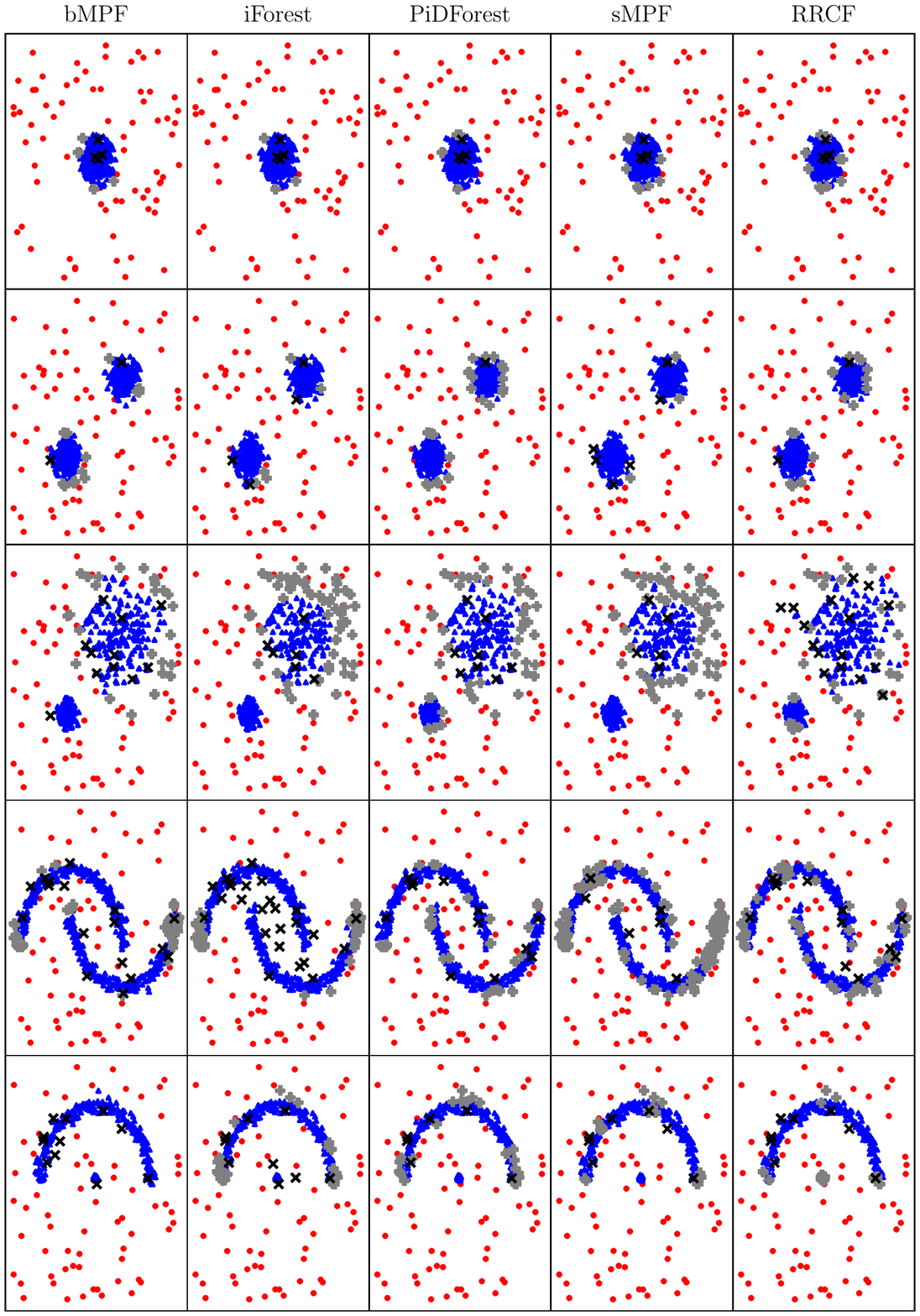}
    \caption{Random Forest Methods on sklearn outlier detection toy
    datasets.
    True positives are in red circles, true negatives in blue triangles,
    false positive in grey $+$, and false negatives in black crosses.
    Black crosses near the modes are often misclassified by all methods;
    these correspond to planted anomalies that lie in the normal region.}
    \label{fig:sklearn-outlier}
\end{figure}

\begin{table}[ht]
\centering
\begin{tabular}{llllll}
\hline
Dataset          & \multicolumn{5}{c}{AUC}                     \\ 
                 & sMPF  & bMPF  & RRCF  & iForest & PidForest \\ \hline 
Single Blob      & 0.963 & 0.966 & 0.972 & 0.964   & 0.963     \\
Two Blobs Tight  & 0.994 & 0.994 & 0.991 & 0.994   & 0.993     \\
Two Blobs Spread & 0.948 & 0.956 & 0.934 & 0.953   & 0.960     \\
Moons            & 0.904 & 0.901 & 0.906 & 0.840   & 0.914     \\
Moon \& Blob     & 0.977 & 0.964 & 0.950 & 0.955   & 0.972     \\ \hline
                 & \multicolumn{5}{c}{AUPRG}                   \\ 
                 & sMPF  & bMPF  & RRCF  & iForest & PidForest \\ \hline
Single Blob      & 0.994 & 0.993 & 0.991 & 0.993   & 0.993     \\
Two Blobs Tight  & 0.999 & 0.998 & 0.998 & 0.999   & 0.999     \\
Two Blobs Spread & 0.986 & 0.989 & 0.982 & 0.984   & 0.990     \\
Moons            & 0.971 & 0.978 & 0.976 & 0.947   & 0.979     \\
Moon \& Blob     & 0.996 & 0.997 & 0.979 & 0.990   & 0.994   \\
\bottomrule
\end{tabular}
\caption{AUC and AUPRG values for \Cref{fig:sklearn-outlier}}
\label{table:synthetic-auprg}
\end{table}

\section{Density Estimation}
We provide 4 synthetic examples to illustrate the use of our proposed models.
A more significant experimental study will be necessary to evaluate the 
efficacy of the models in this context.
The synthetic datasets given below are used to generate an initial sample of 
5000 points, after which a grid is placed over the domain to estimate the 
density.
Further investigation is necessary to understand the efficicacy of both 
Mondrian \Polya{} Forests as density estimators along with a comparison to 
popular methods.

\begin{enumerate}
    \item{Standard Normal: Figure \ref{fig:univariate-density} (left)}
    \item{Univariate Gaussian Mixture: 
    Figure \ref{fig:univariate-density} (right)
    taken from 
    \url{https://scikit-learn.org/stable/auto_examples/neighbors/plot_kde_1d.html}}
    \item{Standard Bivariate Normal: Figure \ref{fig:bivariate-density}
    (left)}
    
    \item{Bivariate Bimodal Mixture: 
    Figure \ref{fig:bivariate-density} (right).
    As in the univariate case, except the 
    covariances are adjusted to alter the shape of the clusters.
    Also the dataset used in \Cref{fig:all-mondrians}.} 
\end{enumerate}

\begin{figure}
    \centering
    \includegraphics[width=0.9\linewidth]{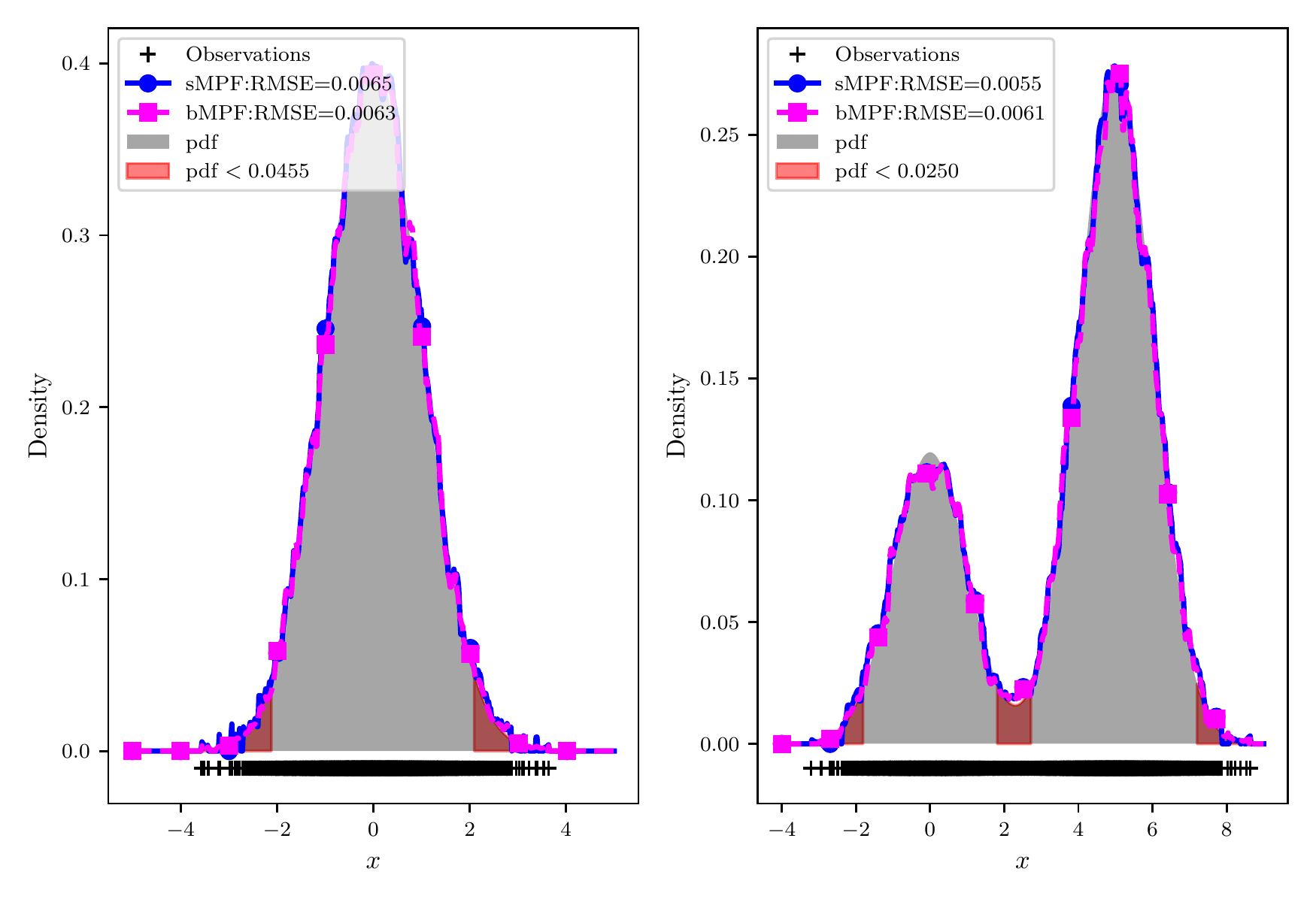}
    \caption{Density estimation on univariate Gaussians}
    \label{fig:univariate-density}
\end{figure}

\begin{figure}
    \centering
    \includegraphics[width=0.9\linewidth]{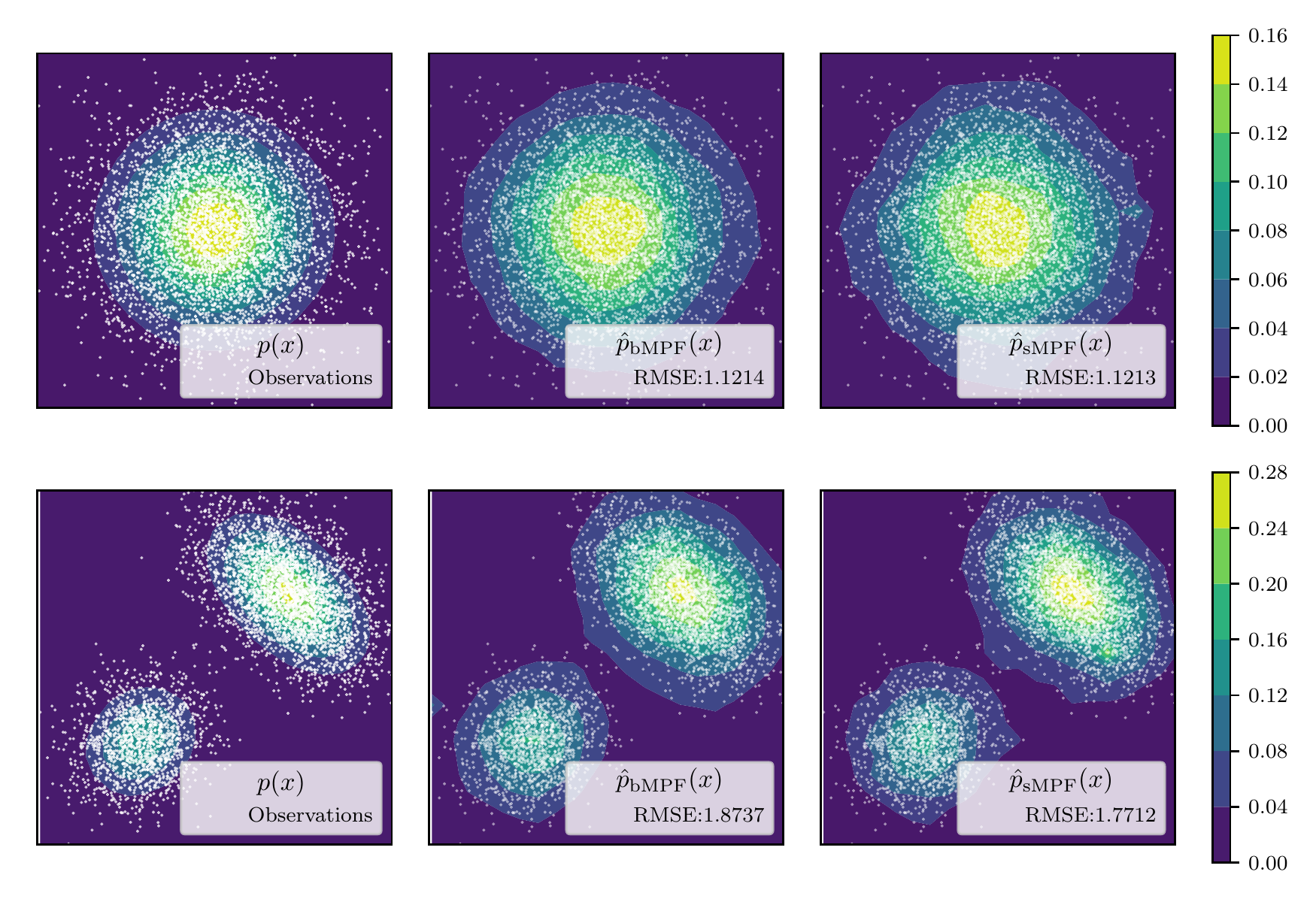}
    \caption{Density estimation on bivariate Gaussians.
    Left-right: True density function, \acrshort{bmpf}, \acrshort{smpf}.}
    \label{fig:bivariate-density}
\end{figure}

\end{document}